\documentclass{article}

\usepackage{microtype}
\usepackage{graphicx}
\usepackage{booktabs} % for professional tables

\usepackage{hyperref}

\usepackage[accepted]{icml2023}

\usepackage{amsmath}
\usepackage{amssymb}
\usepackage{mathtools}
\usepackage{amsthm}
\usepackage{graphicx}

\usepackage[utf8]{inputenc} % allow utf-8 input
\usepackage[T1]{fontenc}    % use 8-bit T1 fonts
\usepackage{hyperref}       % hyperlinks
\usepackage{calc}
\usepackage{url}            % simple URL typesetting
\usepackage{booktabs}       % professional-quality tables
\usepackage{amsfonts}       % blackboard math symbols
\usepackage{nicefrac}       % compact symbols for 1/2, etc.
\usepackage{microtype}      % microtypography
\usepackage{xcolor}         % colors
\usepackage{subcaption}
\usepackage{textcase}
\usepackage{thmtools,thm-restate}

\newlength{\depthofsumsign}
\newlength{\totalheightofsumsign}
\newlength{\heightanddepthofargument}

\makeatletter
\newcommand*{\DivideLengths}[2]{%
  \strip@pt\dimexpr\number\numexpr\number\dimexpr#1\relax*65536/\number\dimexpr#2\relax\relax sp\relax
}
\makeatother

\usepackage{bbm}
\usepackage{booktabs}
\usepackage{pgfplots}
\usepackage{tikz-3dplot}
\usepackage{physics}
\usepackage[outline]{contour} % glow around text
\usetikzlibrary{angles,quotes} % for pic
\contourlength{1.2pt}

\tikzset{>=latex} % for LaTeX arrow head
\usepackage{xcolor}
\usetikzlibrary{fadings}
\usetikzlibrary{arrows.meta}
\tikzset{arl/.style={line width=4pt, {-Latex[left]}, #1}}
\tikzset{arr/.style={line width=4pt, {-Latex[right]}, #1}}

\colorlet{veccol}{green!70!black}
\colorlet{vcol}{green!70!black}
\colorlet{xcol}{blue!85!black}
\colorlet{projcol}{xcol!60}
\colorlet{unitcol}{xcol!60!black!85}
\colorlet{myblue}{blue!70!black}
\colorlet{myred}{red!90!black}
\colorlet{ntk}{red!90!white}
\colorlet{dpk}{blue!70!white}
\colorlet{epkt}{green!40!black}
\colorlet{diffpk}{orange}
\colorlet{epk}{black!70!gray}
\colorlet{mypurple}{blue!50!red!80!black!80}
\tikzstyle{vector}=[->,line width=0.65mm, xcol]
\usetikzlibrary{intersections, pgfplots.fillbetween}
\usepackage{dsfont}

\tdplotsetmaincoords{60}{115}
\pgfplotsset{compat=newest}

\usepackage[capitalize,noabbrev]{cleveref}

\theoremstyle{plain}
\newtheorem{theorem}{Theorem}[section]

\theoremstyle{definition}
\newtheorem{definition}[theorem]{Definition}

\theoremstyle{remark}

\newenvironment{sproof}{%
  \proof}{\endproof}

\newcounter{remcounter}
\newcommand*{\remlabel}[1]{\refstepcounter{remcounter}\theremcounter\label{#1}}

\usepackage[textsize=tiny]{todonotes}

\icmltitlerunning{An Exact Kernel Equivalence for Finite Classification Models }

\begin{document}

\twocolumn[
\icmltitle{An Exact Kernel Equivalence for Finite Classification Models
% (ICML 2023)
}



\icmlsetsymbol{equal}{*}

\begin{icmlauthorlist}
\icmlauthor{Brian Bell}{equal,yyy,comp}
\icmlauthor{Michael Geyer}{equal,yyy,sch}
\icmlauthor{David Glickenstein}{comp}
\icmlauthor{Amanda Fernandez}{sch}
\icmlauthor{Juston Moore}{yyy}
\end{icmlauthorlist}

\icmlaffiliation{yyy}{Los Alamos National Laboratory}
\icmlaffiliation{comp}{University of Arizona}
\icmlaffiliation{sch}{University of Texas San Antonio}

\icmlcorrespondingauthor{Brian Bell}{bwbell@math.arizona.edu}
\icmlcorrespondingauthor{Michael Geyer}{mgeyer@lanl.gov}

\icmlkeywords{Machine Learning, Kernel Machines, Mathematics}

\vskip 0.3in
]



\printAffiliationsAndNotice{\icmlEqualContribution} % otherwise use the standard text.

\begin{abstract}
We explore the equivalence between neural networks and kernel methods by deriving the first exact representation of any finite-size parametric classification model trained with gradient descent as a kernel machine. We compare our exact representation to the well-known Neural Tangent Kernel (NTK) and discuss approximation error relative to the NTK and other non-exact path kernel formulations. We experimentally demonstrate that the kernel can be computed for realistic networks up to machine precision. We use this exact kernel to show that our theoretical contribution can provide useful insights into the predictions made by neural networks, particularly the way in which they generalize.
\end{abstract}

\section{Introduction}

This study investigates the relationship between kernel methods and finite parametric models. To date, interpreting the predictions of complex models, like neural networks, has proven to be challenging. Prior work has shown that the inference-time predictions of a neural network can be exactly written as a sum of independent predictions computed with respect to each training point. We formally show that classification models trained with cross-entropy loss can be exactly formulated as a kernel machine. It is our hope that these new theoretical results will open new research directions in the interpretation of neural network behavior.

\begin{figure}[!ht]
\begin{tikzpicture}[scale=0.72]
\def\bang{72}
\def\ang{35}
\def\lang{15}
\def\mang{57}
\def\sang{69}
\def\kang{90}
\def\tang{87}
\def\dang{34}
\def\xs{1.0}
\def\xa{4}
\def\xe{1.0}
\def\xd{0.75}
\def\xn{0.5}
\def\xi{0.4}
% incoming edge
\coordinate  (s1) at (0,0);
\coordinate  (si1) at ($ (s1) + (\bang:1) $);
\coordinate  (s2) at ($ (si1) + (\bang:\xs) $);
\coordinate  (s3) at ($ (s2) + (\ang:\xa) $);
\coordinate  (si3) at ($(s3) + (\lang:\xs) $);
\coordinate  (s4) at ($(si3) + (\lang:\xs) $);

%\node[fill=black,circle,inner sep=1.9] (s01) at (s1) {};
\path (si1) -- (s2) node [midway, sloped] (elip) {\ldots};
\path (s3) -- (si3) node [midway, sloped] (elip) {\ldots};
\node[fill=black,circle,inner sep=1.9] (s01) at (s1) {};
\node [below right=0mm and 1mm, rotate=\kang-90] (ss1) at (s01) {$w_1(t=0)$};
\node[fill=black,circle,inner sep=1.9] (s02) at (s2) {};
\node [below right=0mm and 1mm, rotate=\kang-90] (ss2) at (s02) {$w_s(t=0)$};
\node[fill=black,circle,inner sep=1.9] (s03) at (s3) {};
\node [below right=-1mm and 1mm, rotate=\kang-90] (ss3) at (s3) {$w_s(t=1)=w_{s+1}(t=0)$}; 
\node[fill=black,circle,inner sep=1.9] (s04) at (s4) {};
%\node[fill=black,circle,inner sep=1.9] (s04) at (s4) {};
\node [below right=-2mm and 1mm, rotate=\kang-90] (ss4) at (s04) {$w_S(t=0)$};
\draw[black, line width=0.4mm, draw opacity=0.3] (s2) -- (s3);
\draw[black, line width=0.4mm, draw opacity=0.3] (s1) -- (si1);
\draw[black, line width=0.4mm, draw opacity=0.3] (si3) -- (s4);

\coordinate (p1) at (s2);%($ (s2)+(\ang:\xi) $);
\coordinate (p1a) at ($ (p1)+(\ang:\xe) $);
\coordinate (p1b) at ($ (p1)+(\mang:\xe) $);
\coordinate (p2) at ($ (p1a)+(\ang:0.6) $);
\coordinate (p2a) at ($ (p2)+(\ang:\xe) $);
\coordinate (p2b) at ($ (p2)+(\sang:\xe) $); 
\coordinate (p2c) at ($ (p2)+(\mang:\xe) $); 
\coordinate (p3) at ($ (p2a)+(\ang:\xi) $);
\coordinate (p3a) at ($ (p3)+(\ang:\xe) $);
\coordinate (p3b) at ($ (p3)+(\tang:\xe) $);
\coordinate (p3c) at ($ (p3)+(\mang:\xe) $);
\coordinate (n0) at ($ (s1)+(\bang:\xn) $);
\coordinate (n1) at ($ (p1)+(\bang:\xn) $);
\coordinate (n2) at ($ (p2)+(\bang:\xn) $);
\coordinate (n3) at ($ (p3)+(\bang:\xn) $);
\coordinate (d0) at ($ (s1)+(\dang:\xd) $);
\coordinate (d1) at ($ (p1)+(\dang:\xd) $);
\coordinate (d2) at ($ (p2)+(\dang:\xd) $);
\coordinate (d3) at ($ (p3)+(\dang:\xd) $);

\draw[name path=pb, black, line width=0.4mm, draw opacity=0.7] (s2) -- (s3);
\draw[black, line width=0.4mm, draw opacity=0.7] (s1) -- (si1);
\path[name path=fb, black, line width=0.4mm, draw opacity=0.7] (s1) -- (s2);
\draw[name path=gb, black, line width=0.4mm, draw opacity=0.7] (si3) -- (s4);  
  \draw [name path=pa, thick , ->, opacity=0.0]
  (p1) .. controls ($ (s2) + (37:2.1) $)  .. ($ (s3) + (125:0.6) $);
  \draw [name path=fa, thick , ->, opacity=0.0]
  (s1) .. controls ($ (s1) + (80:1.1) $)  .. ($ (s2) + (162:0.50) $);
  \draw [name path=ga, thick , ->, opacity=0.0]
  (s3) .. controls ($ (s3) + (20:1) $)  .. ($ (s4) + (105:0.4) $);
\tikzfillbetween[of=pa and pb]{orange, opacity=0.2};
\tikzfillbetween[of=fa and fb]{orange, opacity=0.2}; 
\tikzfillbetween[of=ga and gb]{orange, opacity=0.2};
    
% \draw[vector, ->, ntk] (s1) -- (n0) node[scale=1,above left=-3mm and 3mm, draw opacity=0.5] {$\vu{x}$};
% \draw[vector, ->, dpk] (s1) -- (d0) node[scale=1,above left=-3mm and 3mm, draw opacity=0.5] {$\vu{x}$};
\draw[vector, ->, dpk] (p1) -- (p1b) node[scale=1,above left=-5mm and 2mm, draw opacity=0.5] {$\nabla f_{w_s(t=0)}(x)$};
\begin{scope}
\clip (p1) -- (p1b) -- ($(p1b) + (90+\mang:0.1) $) -- ($ (p1) + (90+\mang:0.1) $);

\draw[vector, ->, epkt] (p1) -- (p1b) node[scale=1,above left=-5mm and 2mm, draw opacity=0.5] {$\nabla f_{w_s(t=0)}(x)$};

\end{scope}
\draw[vector, ->, epk] (p1) -- (p1a) node[scale=1,below right=1mm and -2mm, draw opacity=0.5, rotate=\kang-90] {$\nabla f_{w_s(t=0)}(X)$};
\draw[vector, ->, epkt] (p2) -- (p2b) node[scale=1,above left=-5mm and 2mm, draw opacity=0.5] {$\nabla f_{w_s(t=0.4)}(x)$};
\draw[vector, ->, dpk] (p2) -- (p2c) node[scale=1,above left=-5mm and 2mm, draw opacity=0.5] {};
\draw[vector, ->, epk] (p2) -- (p2a) node[scale=1,below right=3mm and -4mm, draw opacity=0.5, rotate=\kang-90] {$\nabla f_{w_s(t=0)}(X)$};
\draw[vector, ->, epkt] (p3) -- (p3b) node[scale=1,above left=-5mm and 2mm, draw opacity=0.5] {$\nabla f_{w_s(t=0.7)}(x)$};
\draw[vector, ->, dpk] (p3) -- (p3c) node[scale=1,above left=-5mm and 2mm, draw opacity=0.5] {};
\draw[vector, ->, epk] (p3) -- (p3a) node[scale=1,below right=4mm and -6mm, draw opacity=0.5, rotate=\kang-90] {$\nabla f_{w_s(t=0)}(X)$};
\draw[->,line width=0.2mm, xcol, diffpk] (p2c) -- (p2b) node[scale=1,below right=1mm and -2mm, draw opacity=0.5] {};
\draw[->,line width=0.2mm, xcol, diffpk] (p3c) -- (p3b) node[scale=1,below right=1mm and -2mm, draw opacity=0.5] {};

  %\draw[vector,<->,unitcol]
  %  (v3) node[scale=1,above left=-3mm and 3mm] {$\vu{x}$} -- (v1) --
  %  (v2) node[scale=1,below=2,below left=1mm and 0mm] {$\vu{y}$};
    
    %   \draw[vector,<->,unitcol]
    % (v6) node[scale=1,above left=-3mm and 3mm] {$\vu{x}$} -- (v4) --
    % (v5) node[scale=1,below=2,below left=1mm and 0mm] {$\vu{y}$};
    %       \draw[vector,<->,unitcol]
    % (v9) node[scale=1,above left=-3mm and 3mm] {$\vu{x}$} -- (v7) --
    % (v8) node[scale=1,below=2,below left=1mm and 0mm] {$\vu{y}$};

% active step
% % outgoing edge
%   \def\ul{0.52}
%   \def\R{2.6}
%   \def\ang{28}
%   \coordinate (O) at (0,0);
%   \coordinate (R) at (\ang:\R);
%   \coordinate (X) at ({\R*cos(\ang)},0);
%   \coordinate (Y) at (0,{\R*sin(\ang)});
%   \node[fill=black,circle,inner sep=1.9] (O') at (O) {};
%   \node[fill=black,circle,inner sep=1.9] (R') at (R) {};
%   \node[above right=-2] at (R') {$(x,y)$};
%   \draw[<->,line width=0.9] %very thick
%     ({1.2*\R*cos(\ang)},0) -- (O) -- (0,{1.3*\R*sin(\ang)});
%   \draw[projcol,dashed] (X) -- (R);
%   \draw[black, thick] (O') -- (R');
%   \draw[projcol,dashed] (Y) -- (R);
%   %\draw[vector] (O) -- (R') node[midway,left=5,above right=0] {$\vb{r}$};
%   \draw[vector,<->,unitcol]
%     (\ul,0) node[scale=1,left=2,below left=0] {$\vu{x}$} -- (O) --
%     (0,\ul) node[scale=1,below=2,below left=0] {$\vu{y}$};
%   \draw pic[->,thick,"$\theta$",draw=black,angle radius=26,angle eccentricity=1.3]
%     {angle = X--O--R};
%   \draw[thick] (X)++(0,0.1) --++ (0,-0.2) node[scale=0.9,below=-1] {$x = r\cos\theta$};
%   \draw[thick] (Y)++(0.1,0) --++ (-0.2,0) node[scale=0.9,left] {$y = r\sin\theta$};
\end{tikzpicture}
\caption{Comparison of test gradients used by Discrete Path Kernel (DPK) from prior work (Blue) and the Exact Path Kernel (EPK) proposed in this work (green) versus total training vectors (black) used for both kernel formulations along a discrete training path with $S$ steps. Orange shading indicates cosine error of DPK test gradients versus EPK test gradients shown in practice in Fig.~\ref{fig:error}. }
\label{fig:vecs}
\end{figure}
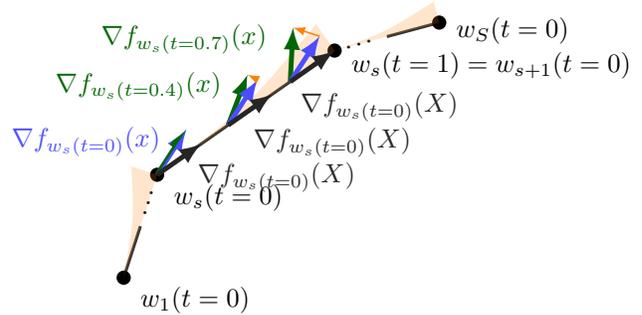

% \begin{figure}[!ht]
% % \begin{tikzpicture}
% % \node
% % \end{tikzpicture}
% \end{figure}
\begin{figure}[!ht]
        \centering
        % \begin{minipage}{0.5\textwidth}
        \includegraphics[width=1.01\linewidth]{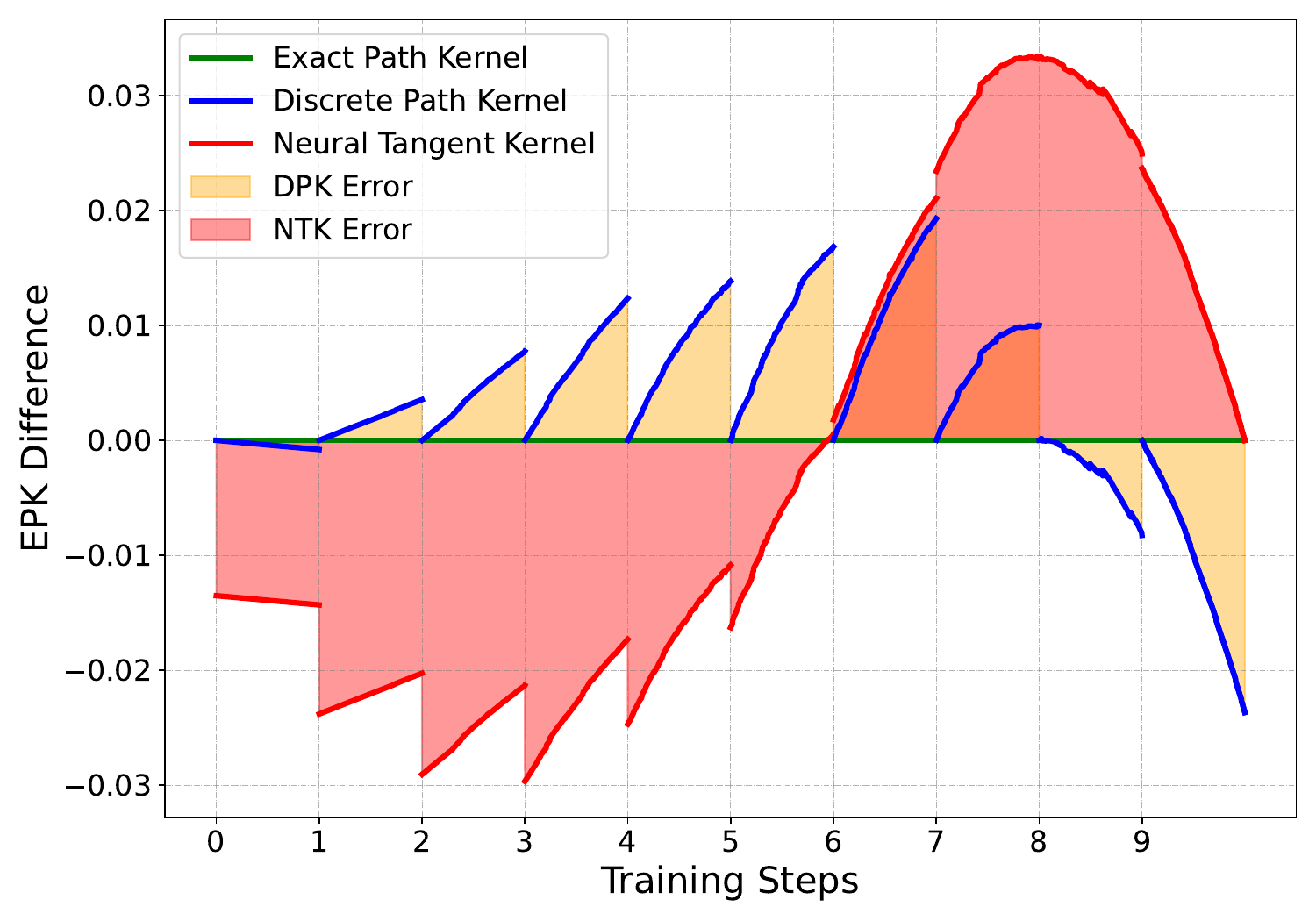}
        % \end{minipage}

        \caption{Measurement of gradient alignment on test points across the training path. The EPK is used as a frame of reference. The y-axis is exactly the difference between the EPK and other representations. For example $EPK-DPK = \langle \phi_{s,t}(X), \phi_{s,t}(x) - \phi_{s,0}(x) \rangle$ (See Definition 3.4). Shaded regions indicate total accumulated error. Note: this is measuring an angle of error in weight space; therefore, equivalent positive and negative error will not result in zero error.}
        \label{fig:error}
\end{figure}

There has recently been a surge of interest in the connection between neural networks and kernel methods~\cite{bietti2019bias, du2019graphntk, tancik2020fourierfeatures, abdar2021uq, geifman2020similarity, chen2020generalized, alemohammad2021recurrent}. Much of this work has been motivated by the the neural tangent kernel (NTK), which describes the training dynamics of neural networks in the infinite limit of network width~\cite{jacot2018neural}.
We argue that many intriguing behaviors arise in the \emph{finite} parameter regime~\cite{DBLP:conf/nips/BubeckS21}. 
All prior works, to the best of our knowledge, appeal to discrete approximations of the kernel corresponding to a neural network. 
Specifically, prior approaches are derived under the assumption that training step size is small enough to guarantee close approximation of a gradient flow
%. These approximations have been speculated to open up parametric models trained with gradient descent, including artificial neural networks (ANNs), to many theoretical tools available to kernel methods
~\cite{ghojogh2021, shawe2004kernel, zhao2005extracting}.

In this work, we show that the simplifying assumptions used in prior works (i.e. infinite network width and infinitesimal gradient descent steps) are not necessary. Our \textbf{Exact Path Kernel (EPK)} provides the first, exact method to study the behavior of finite-sized neural networks used for classification.
Previous results are limited in application ~\cite{incudini2022quantum} due to dependence of the kernel on test data unless strong conditions are imposed on the training process as by ~\cite{chen2021equivalence}. We show, however, that the training step sizes used in practice do not closely follow this gradient flow, introducing significant error into all prior approaches (Figure~\ref{fig:error}).
%use of such tools is highly dependent on the assumption that ANN training is a faithful discrete approximation of the smooth path kernel at practical step sizes. The accuracy of this assumption depends on difficult measurements of convergence and error.

Our experimental results build on prior studies attempting to evaluate empirical properties of the kernels corresponding to finite neural networks ~\cite{DBLP:conf/iclr/LeeBNSPS18, chen2021equivalence}. While the properties of infinite neural networks are fairly well understood~\cite{neal1996priors}, we find that the kernels learned by finite neural networks have non-intuitive properties that may explain the failures of modern neural networks on important tasks such as robust classification and calibration on out-of-distribution data.

This paper makes the following significant theoretical and experimental contributions:
\begin{enumerate}
    \item We prove that finite-sized neural networks trained with finite-sized gradient descent steps and cross-entropy loss can be exactly represented as kernel machines using the EPK. Our derivation incorporates a previously-proposed path kernel, but extends this method to account for practical training procedures~\cite{domingos2020every, chen2021equivalence}.
  
    \item We demonstrate that it is computationally tractable to estimate the kernel underlying a neural network classifier, including for small convolutional computer vision models.
    % \item We estimate a kernel corresponding to both a toy-sized neural network and a convolutional classifier on MNIST up to machine precision.
    \item We compute Gram matrices using the EPK and use them to illuminate prior theory of neural networks and their understanding of uncertainty. 
    \item We employ Gaussian processes to compute the covariance of a neural network's logits and show that this reiterates previously observed shortcomings of neural network generalization.
\end{enumerate}

\section{Related Work}
% Interpreting and understanding 

% The neural tangent kernel (NTK) ~\cite{he2020bayesian} has received significant attention recently in light of a growing body of work relating parametric gradient models in infinite width with Gaussian processes (GPs) and also the NTK. 
Fundamentally, the neural tangent kernel (NTK) is rooted in the concept that all information necessary to represent a parametric model is stored in the Hilbert space occupied by the model's weight gradients up to a constant factor. 
This is very well supported in infinite width ~\cite{jacot2018neural}. 
In this setting, it has been shown that neural networks are equivalent to support vector machines, drawing a connection to maximum margin classifiers ~\cite{chen2021equivalence, chizat2020maxmargin}.
Similarly, Shah et al. demonstrate that this maximum margin classifier exists in Wasserstien space; however, they also show that model gradients may not contain the required information to represent this ~\cite{shah2021input}.
% However, discrete approximation of an infinite-width kernel is very limiting.

% This formulation derives a kernel for finite-parametric models however it relies on a continuous integration over a gradient flow, while real-world models are trained by discrete Forward Euler steps through a gradient field.
% The exploration of discrete approximations has been a key focus in addressing the challenges associated with neural networks. 
The correspondence between kernel machines and parametric models trained by gradient descent has been previously developed in the case of a continuous training path (i.e. the limit as gradient descent step size $\varepsilon \to 0$)
% One notable approach is the formulation of the continuous path kernel, which aims to derive a kernel for finite-parametric models
~\cite{domingos2020}. We will refer to the previous finite approximation of this kernel as the Discrete Path Kernel (DPK).
However, a limitation of this formulation is its reliance on a continuous integration over a gradient flow, which differs from the discrete forward Euler steps employed in real-world model training. 
This discrepancy raises concerns regarding the applicability of the continuous path kernel to practical scenarios ~\cite{incudini2022quantum}.
Moreover, the formulation of the sample weights and bias term in the DPK depends on its test points. Chen et al. propose that this can be addressed, in part, by imposing restrictions on the loss function used for training, but did not entirely disentangle the kernel formulation from sample importance weights on training points ~\cite{chen2021equivalence}.

We address the limitations of \citet{domingos2020} and \citet{chen2021equivalence} in Subsection %~\remref{rem5}
~\ref{subsec:disc}. By default, their approach produces a system which can be viewed as an ensemble of kernel machines, but without a single aggregated kernel which can be analyzed directly. ~\citet{chen2021equivalence} propose that the resulting sum over kernel machines can be formulated as a kernel machine so long as the sign of the gradient of the loss stays constant through training; however, we show that this is not necessarily a sufficient restriction. Instead, their formulation leads to one of several non-symmetric functions which can serve as a surrogate to replicate a given models behavior, but without retaining properties of a kernel.

\section{Theoretical Results}

Our goal is to show an equivalence between any given finite parametric model trained with gradient descent $f_w(x)$  (e.g. neural networks) and a kernel based prediction that we construct. We define this equivalence in terms of the output of the parametric model $f_w(x)$ and our kernel method in the sense that they form identical maps from input to output. In the specific case of neural network classification models, we consider the mapping $f_w(x)$ to include all layers of the neural network up to and including the log-softmax activation function. Formally:
\begin{definition}
A {kernel} is a function of two variables which is symmetric and positive semi-definite. 
\end{definition}

\begin{definition}
Given a Hilbert space $X$, a test point $x \in X$, and a training set $X_T = \{x_1,x_2,...x_n\} \subset X$ indexed by $I$, a \emph{Kernel Machine} is a model characterized by 
\begin{align}
    \text{K}(x) = b + \sum_{i\in I} a_i k(x,x_i)
\end{align}
where the $a_i \in \mathbb{R}$ do not depend on $x$, $b \in \mathbb{R}$ is a constant, and $k$ is a kernel. ~\cite{rasmussen2006gaussian}

By Mercer's theorem ~\cite{ghojogh2021} a kernel can be produced by composing an inner product on a Hilbert space with a mapping $\phi$ from the space of data into the chosen Hilbert space.
We use this property to construct a kernel machine of the following form.
\begin{align}
    \text{K}(x) = b + \sum_{i\in I} a_i \langle \phi(x), \phi(x_i) \rangle
\end{align}
\end{definition}
Where $\phi$ is a function mapping input data into the weight space via gradients. Our $\phi$ will additionally differentiate between test and training points to resolve a discontinuity that arises under discrete training. 

\subsection{Exact Path Kernels}

 We first derive a kernel which is an exact representation of the change in model output over one training step, and then compose our final representation by summing along the finitely many steps.
Models trained by gradient descent can be characterized by a discrete set of intermediate states in the space of their parameters.
These discrete states are often considered to be an estimation of the gradient flow, however in practical settings where $\epsilon \not \rightarrow 0$ these discrete states differ from the true gradient flow.
Our primary theoretical contribution is an algorithm which accounts for this difference by observing the true path the model followed during training.
Here we consider the training dynamics of practical gradient descent steps by integrating a discrete path for weights whose states differ from the gradient flow induced by the training set.
% As such, the model's training dynamics do not exactly follow the the gradient flow defined by the model's loss function (i.e. the $\epsilon \rightarrow 0$ limit in forward Euler). 
%These states are commonly computed by gradient descent.

\textbf{Gradient Along Training Path vs Gradient Field:}
In order to compute the EPK, gradients on training data must serve two purposes. 
First, they are the reference points for comparison (via inner product) with test points. 
Second, they determine the path of the model in weight space. 
% % For a continuous path kernel which follows a gradient flow, gradients on the training data exactly match (determine) the path of the parameters through the gradient field.
% This would allow us to simply evaluate the gradient of the training data directly.
% This means that that for every point, we can simply evaluate the gradient of the training data directly.
% Unfortunately, the path followed in practice is not the gradient flow.
In practice, the path followed during gradient descent does not match the gradient field exactly. 
Instead, the gradient used to move the state of the model forward during training is only computed for finitely many discrete weight states of the model.
In order to produce a path kernel, we must \textit{continuously} compare the model's gradient at test points with \textit{fixed} training gradients along each discrete training step $s$ whose weights we we interpolate linearly by $w_s(t) = w_s - t(w_s - w_{s+1})$. We will do this by integrating across the gradient field induced by test points, but holding each training gradient fixed along the entire discrete step taken. This creates an asymmetry, where test gradients are being measured continuously but the training gradients are being measured discretely (see Figure~\ref{fig:vecs}).
% One problem encountered when constructing a path kernel for discrete steps is the which requires addressing is the different ways which this learned kernel treats its training points compared to all other inputs.
% The path taken during training is dependent on the combination of model, training data, learning algorithm and loss function.
% Because the final path includes dependence on training data, the final kernel that represents this path will also depend on the training data.

To account for this asymmetry in representation, we will redefine our data using an indicator to separate training points from all other points in the input space.
\begin{definition}
\label{fpm}
Let $X$ be two copies of a Hilbert space $H$ with indices $0$ and $1$ so that $X = H \times \{0,1\}$. We will write $x \in H \times \{0,1\}$ so that $x = (x_H, x_I)$ (For brevity, we will omit writing $_H$ and assume each of the following functions defined on $H$ will use $x_H$ and $x_I$ will be a hidden indicator).
Let $ f_{w}$ be a differentiable function on $H$ parameterized by $w \in \mathbb{R}^d$. Let $X_T = \{(x_i, 1)\}_{i=1}^M$ be a finite subset of $X$ of size $M$ with corresponding observations $Y_T = \{y_{x_i}\}_{i=1}^M$ with initial parameters $w_0$ so that there is a constant $b \in \mathbb{R}$ such that for all $x$, $ f_{w_0}(x) = b$. Let $L$ be a differentiable loss function of two values which maps $(f(x), y_x)$ into the positive real numbers. Starting with $f_{w_0}$, let $\{w_s\}$ be the sequence of points attained by $N$ forward Euler steps of fixed size $\varepsilon$ so that $w_{s+1} = w_{s} - \varepsilon \nabla L(f(X_T), Y_T)$. Let $x \in H \times \{0\}$ be arbitrary and within the domain of $f_w$ for every $w$. Then $f_{w_s(t)}$ is a \emph{finite parametric gradient model (FPGM)}. 
\end{definition}

\begin{definition}
\label{epk}

Let $f_{w_s(t)}$ be an FPGM with all corresponding assumptions. Then, for a given training step $s$, the \emph{exact path kernel} (EPK) can be written  
\begin{equation}
 K_{\text{EPK}}(x, x', s) = \int_0^1\langle \phi_{s,t}(x), \phi_{s,t}(x')\rangle dt
 \label{eq2}
\end{equation}
where
\begin{align}
% a_{i, s} &= -\varepsilon  \dfrac{\partial L(f_{w_s(0)}(x_i),  y_i)}{\partial f_i} \in \mathbb{R} \\
\phi_{s, t}(x) &=  \nabla_w f_{w_s(t,x)} (x)\\
w_s(t) &= w_s - t(w_s - w_{s+1})\\
w_s(t,x) &= \begin{cases} w_s(0), & \text{if } x_I = 1\\ w_s(t), & \text{if } x_I = 0 \end{cases}
% b &= f_{w_0}(x) 
\end{align}
\textbf{Note:} $\phi$ is deciding whether to select a continuously or discrete gradient based on whether the data is from the training or testing copy of the Hilbert space $H$. This is due to the inherent asymmetry that is apparent from the derivation of this kernel (see Appendix section~\ref{proof:eker}). This choice avoids potential discontinuity in the kernel output when a test set happens to contain training points. 
\end{definition}
\begin{restatable}{lemma}{ker}
The exact path kernel (EPK) is a kernel.
\end{restatable}
\begin{restatable}[Exact Kernel Ensemble Representation]{theorem}{eker}
\label{thm:eker}
A model $f_{w_N}$ trained using discrete steps matching the conditions of the exact path kernel has the following exact representation as an ensemble of $N$ kernel machines:
\begin{equation}
f_{w_N} = \text{KE}(x) :=  \sum_{s = 1}^N \sum_{i = 1}^{M} a_{i,s} K_{\text{EPK}}(x, x', s) + b
\label{ensemble}
\end{equation}
where
\begin{align}
a_{i, s} &= -\varepsilon  \dfrac{d L(f_{w_s(0)}(x_i),  y_i)}{d f_{w_s(0)}(x_i)} \\
% \phi_{s, t}(x) &=  \nabla_w f_{w_s(t,x)} (x)\\
% w_s(t,x) &= \begin{cases} w_s, & \text{if } x_I = 1\\ w_s(t), & \text{if } x_I = 0 \end{cases}
b &= f_{w_0}(x)
\end{align}
\end{restatable}

\begin{sproof}
Assuming the theorem hypothesis, we'll measure the change in model output as we interpolate across each training step $s$ by measuring the change in model state along a linear parametrization $w_s(t) = w_s - t(w_s - w_{s+1})$. We will let $d$ denote the number of parameters of $f_w$. For brevity, we define $L(x_i, y_i)= l(f_{w_s(0)}(x_i),  y_i)$ where $l$ is the loss function used to train the model.
\begin{align}
    \dfrac{d \hat y}{dt} &= \sum_{j = 1}^{d} \dfrac{d \hat y}{\partial w_j} \dfrac{d w_j}{dt}\\
&= \sum_{j = 1}^{d} \dfrac{d f_{w_s(t)}(x)}{\partial w_j} \left(-\varepsilon \sum_{i = 1}^{M}\dfrac{\partial L(x_i, y_i)}{\partial f_{w_s(0)}(x_i)}\dfrac{\partial f_{w_s(0)}(x_i)}{\partial w_j}\right) \label{eq11}
\end{align}
We use fundamental theorem of calculus to integrate this equation from step $s$ to  step $s+1$ and then add up across all steps. See Appendix~\ref{proof:eker} for the full proof.
\end{sproof}

\textbf{Remark ~\remlabel{rem:init}} Note that in this formulation, $b$ depends on the test point $x$.
In order to ensure information is not being leaked from the kernel into this bias term the model $f$ must have constant output for all input. 
When relaxing this property, to allow for models that have a non-constant starting output, but still requiring $b$ to remain constant, we note that this representation ceases to be exact for all $x$.
The resulting approximate representation has logit error bounded by its initial bias which can be chosen as $b = \text{mean}(f_{w_0(0)}(X_T))$.
Starting bias can be minimized by starting with small parameter values which will be out-weighed by contributions from training.
In practice, we sidestep this issue by initializing all weights in the final layer to $0$, resulting in $b=\text{log}(\text{softmax}(0))$, thus removing $b$'s dependence on $x$.

\textbf{Remark ~\remlabel{rem:exact}} 
The exactness of this proof hinges on the \emph{separate} measurement of how the model's parameters change.
The gradients on training data, which are fixed from one step to the next, measure how the parameters are changing.
This is opposed to the gradients on test data, which are \textit{not} fixed and vary with time.
These measure a continuous gradient field for a given point.
We are using interpolation as a way to measure the difference between the step-wise linear training path and the continuous loss gradient field. 

\newcommand{\pluseq}{\mathrel{+}=}
\begin{algorithm*}[h]
    \caption{Exact Path Kernel: Given a training set $(X, Y)$ with $M$ data points, a testing point $x$ and $N$ weight states $\{w_0, w_1 ... w_N\}$, the kernel machine corresponding to the exact path kernel can be calculated for a model with $W$ weights and $K$ outputs. We estimate the integral across test points by calculating the Riemann sum with sufficient steps ($T$) to achieve machine precision. For loss functions that do not have constant gradient values throughout training, this algorithm produces an ensemble of kernel machines.}
    \label{alg:exact}
\begin{algorithmic}
    % \STATE {\bfseries Input:} $(X_T, Y_T)$, $x$, $w_s \in S$
    \STATE $b = f(w_0, x)$
    \FOR{$s=0$ \textbf{to} $N$}
    \STATE $J^{X} = \nabla_{w} f_{w_s(0)}(X)$ \hfill \COMMENT{Jacobian of training point outputs w.r.t model weights $[M \times K \times W]$ }
    \FOR{t \textbf{from} 0 \textbf{to} 1 \textbf{with step} 1/T}
    \STATE $w_s(t) = w_s + t(w_{s+1} - w_s)$
    \STATE $J^{x} \pluseq \dfrac{1}{T} \nabla_{w} f_{w_s(t)}(x)$ \hfill \COMMENT{Jacobian of testing point output w.r.t model weights averaged across $T$ steps $ [K \times W]$}
    \ENDFOR
    \STATE $G_{ijk} =  \sum_w J_{ijw}^{X} J_{kw}^x$ \hfill \COMMENT{Inner product on the weight space, this is the kernel value $[M \times K \times K]$}
    \STATE $L' = \nabla_{f}L(f_{w_s(0)}(X), Y)$ \hfill \COMMENT{Jacobian of loss w.r.t model output of training points $[M \times K]$}
    \STATE $P^s_{ik} = \sum_j L'_{ij} G_{ijk}$ \hfill \COMMENT{Inner product of kernel value scaled by loss gradients $[M \times K$]}
    \ENDFOR
    \STATE $\mathcal{P}_{sik} = \{P^0, P^1, ..., P^N\}$ \hfill \COMMENT{Stack values across all training steps $[N \times M \times K]$}
    \STATE $\hat p = -\varepsilon \dfrac{1}{M} \sum_s \sum_i \mathcal{P}_{sik} + b$ \hfill \COMMENT{Sum across training steps and average across training points for final prediction $[K]$}
\end{algorithmic}
% \caption{caption}
\end{algorithm*}

\begin{restatable}[Exact Kernel Machine Reduction]{theorem}{ekr}
\label{thm:ekr}
Let $\nabla L(f(w_{s}(x), y)$ be constant  across steps $s$, $(a_{i,s}) = (a_{i,0})$. Let the kernel across all $N$ steps be defined as $K_{\text{NEPK}}(x,x') = \sum_{s = 1}^N a_{i,0} K_{\text{EPK}}(x, x', s)$ Then the exact kernel ensemble representation for $f_{w_N}$ can be reduced exactly to the kernel machine representation:
\begin{equation}
f_{w_N}(x) = \text{KM}(x) := b + \sum_{i = 1}^{M} a_{i,0} K_{\text{NEPK}}(x,x')
\label{exact}
\end{equation}
\end{restatable}
See Appendix~\ref{proof:ekmr} for full proof. By combining theorems ~\ref{thm:eker} and ~\ref{thm:ekr}, we can construct an exact kernel machine representation for any arbitrary parameterized model trained by gradient descent which satisfies the additional property of having constant loss across training steps (e.g. any ANN using catagorical cross-entropy loss (CCE) for classification). This representation will produce exactly identical output to the model across  the model's entire domain. This establishes exact kernel-neural equivalence for classification ANNs. Furthermore, Theorem ~\ref{thm:eker} establishes an exact kernel ensemble representation without limitation to models using loss functions with constant derivatives across steps. It remains an open problem to determine other conditions under which this ensemble may be reduced to a single kernel representation.  

\subsection{Discussion}
%\textbf{Remark \remlabel{rem0}} 

%\textbf{Remark \remlabel{rem1}} 
$\phi_{s,t}(x)$ depends on both $s$ and $t$, which is non-standard but valid, however an important consequence of this mapping is that the output of this representation is not guaranteed to be continuous. This discontinuity is exactly measuring the error between the model along the exact path compared with the gradient flow for each step. 

We can write another function $k'$ which is continuous but not symmetric, yet still produces an exact representation:
\begin{align}
k'(x, x') = \langle \nabla_w f_{w_s(t)}(x), \nabla_w f_{w_s(0)}(x')\rangle
\end{align}
The resulting function is a valid kernel if and only if for every $s$ and every $x$, 
\begin{align}
\label{eq:cond}
    \int_0^1 \nabla_w f_{w_s(t)}(x)dt = \nabla_w f_{w_s(0)}(x)
\end{align}

%\textbf{Remark \remlabel{rem3}} 
We note that since $f$ is being trained using forward Euler, we can write:
\begin{align}
    \dfrac{\partial w_s(t)}{dt} &= -\varepsilon \nabla_w L(f_{w_s(0)}(x_i), y_i) \label{dstep}% = -\varepsilon \sum_{j = 1}^{d} \dfrac{\partial L(f_{w_s(0)}(x_i),  y_i)}{\partial w_j} \label{rem3}
\end{align}
In other words, our parameterization of this step depends on the step size $\varepsilon$ and as $\varepsilon \to 0$, we have 
\begin{align}
    \int_0^1 \nabla_w f_{w_{s}(t)}(x)dt \approx \nabla_w f_{w_s(0)}(x)
\end{align}
In particular, given a model $f$ that admits a Lipshitz constant $K$ this approximation has error bounded by $\varepsilon K$ and a proof of this convergence is direct. 
This demonstrates that the asymmetry of this function is exactly measuring the disagreement between the discrete steps taken during training with the gradient field. 
This function is one of several subjects for further study, particularly in the context of Gaussian processes whereby the asymmetric Gram matrix corresponding with this function can stand in for a covariance matrix. It may be that the not-symmetric analogue of the covariance in this case has physical meaning relative to uncertainty.

\subsection{Independence from Optimization Scheme}
We can see that by changing equation ~\ref{dstep} we can produce an exact representation for any first order discrete optimization scheme that can be written in terms of model gradients aggregated across subsets of training data. This could include backward Euler, leapfrog, and any variation of adaptive step sizes. This includes stochastic gradient descent, and other forms of subsampling (for which the training sums need only be taken over each sample). One caveat is adversarial training, whereby the $a_i$ are now sampling a measure over the continuum of adversarial images. We can write this exactly, however computation will require approximation across the measure. Modification of this kernel for higher order optimization schemes remains an open problem.

%\textbf{Remark \remlabel{rem2}} 

\begin{figure*}[!ht]
    \centering
        \includegraphics[width=0.3\textwidth]{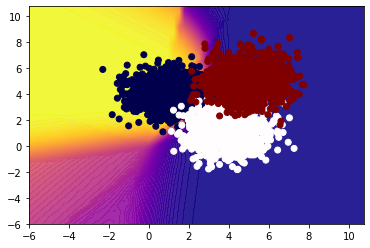}\includegraphics[width=0.3\textwidth]{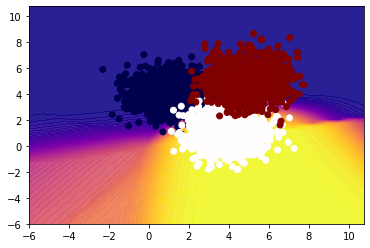}\includegraphics[width=0.3\textwidth]{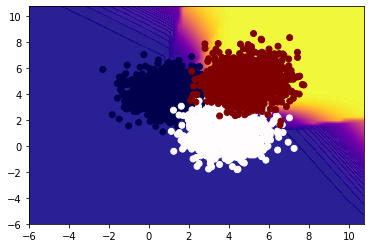}
    \caption{Updated predictions with kernel $a_i$ updated via gradient descent with training data overlaid for classes 1 (left), 2 (middle), and 3 (right). The high prediction confidence in regions far from training points demonstrates that the learned kernel is non-stationary.}
    \label{fig:svm}
\end{figure*}
\subsection{Ensemble Reduction}
%\textbf{Remark \remlabel{rem4}} 
In order to reduce the ensemble representation of Equation ~\eqref{ensemble} to the kernel representation of Equation ~\eqref{exact}, we require that the sum over steps still retain the properties of the kernel (symmetry and positive semi-definiteness). In particular we require that for every subset of the training data ${x_i}$ and arbitrary ${\alpha_i}$ and ${\alpha_j}$, we have
\begin{align}
    \sum_{i=1}^n\sum_{j=1}^n \sum_{l=1}^M \sum_{s=1}^N \alpha_i \alpha_j a_{l, s}\int_{0}^1 K_{\text{EPK}}(x_i,x_j) dt \geq 0
\end{align}
A sufficient condition for this reduction is that the gradient of the loss function does not change throughout training. This is the case for categorical cross-entropy where labels are in $\{0,1\}$. In fact, in this specific context the gradient of the loss function does not depend on $f(x)$, and are fully determined by the ground truth label, making the gradient of the cross-entropy loss a constant value throughout training (See Appendix section ~\ref{proof:ekmr}). Showing the positive-definiteness of more general loss functions (e.g. mean squared error loss) will likely require additional regularity conditions on the training path, and is left as future work.
% There are other conditions which may be imposed in order to guarantee this reduction, and it remains to be studied whether this is unconditionally true for certain training paths.
%  We demonstrate that constant sign of the loss gradient is a sufficient but not necessary condition for positive-semi-definiteness of the kernel.

\subsection{Prior Work}
\label{subsec:disc}
%\textbf{Remark \remlabel{rem5}} 
Constant sign loss functions have been previously studied by Chen et al. ~\cite{chen2021equivalence}, however the kernel that they derive for a finite-width case is of the form
\begin{align}
    K(x,x_i) =  \int_0^T |\nabla_f L(f_t(x_i), y_i)| \langle \nabla_w f_t(x), \nabla_w f_t(x_i) \rangle dt
\end{align}
The summation across these terms satisfies the positive semi-definite requirement of a kernel, however the weight $|\nabla L(f_t(x_i), y_i)|$ depends on $x_i$ which is one of the two inputs. This makes the resulting function $K(x,x_i)$ asymmetric and therefore not a kernel.
% In this formulation, weight $|\nabla L(f_t(x_i), y_i)|$ depends on $x_i$ which is one of the two inputs. This makes the resulting function $k$ asymmetric and therefore not a kernel.

%\textbf{Remark \remlabel{rem6}} 
\subsection{Uniqueness}
Uniqueness of this kernel is not guaranteed. 
The mapping from paths in gradient space to kernels is in fact a function, meaning that each finite continuous path has a unique exact kernel representation of the form described above. 
However, this function is not necessarily onto the set of all possible kernels. 
This is evident from the existence of kernels for which representation by a finite parametric function is impossible.
Nor is this function necessarily one-to-one since there is a continuous manifold of equivalent parameter configurations for neural networks.
For a given training path, we can pick another path of equivalent configurations whose gradients will be separated by some constant $\delta > 0$.
The resulting kernel evaluation along this alternate path will be exactly equivalent to the first, despite being a unique path. 
We also note that the linear path $l_2$ interpolation is not the only valid path between two discrete points in weight space.
Following the changes in model weights along a path defined by Manhattan Distance is equally valid and will produce a kernel machine with equivalent outputs.
It remains an open problem to compute paths from two different starting points which both satisfy the constant bias condition from Definition~\eqref{epk} which both converge to the same final parameter configuration and define different kernels.

\section{Experimental Results}
    % \begin{figure}
    %     \centering
    %     \begin{minipage}{0.45\textwidth}
    %         \centering
    %         \includegraphics[width=0.95\linewidth]{figures/sample_dataset.pdf}
    %     \end{minipage}
    %     \begin{minipage}{0.45\textwidth}
    %         \centering
    %         \includegraphics[width=0.95\linewidth]{figures/model_kernel_predictions.pdf}
    %     \end{minipage}
    %     \caption{On the left is a 2d dataset of points sampled from Gaussians with different means. Specifically, class A is normally distributed with $\mu = \left[1, 4\right]$ and $\sigma^2 = 1$ while class B is $\mu = \left[4, 1\right]$ and $\sigma^2 = 1$. 2000 data points were sampled for each class. These values were chosen arbitrarily to provide separation with a limited amount of overlap. On the right is the prediction similarity between the kernel and the original model. This demonstrates that our kernel formulation accurately represents the trained network.}
    %     \label{fig:sample_data}
    % \end{figure}
    Our first experiments test the kernel formulation on a dataset 
    which can be visualized in 2d. These experiments serve as a sanity check
    and provide an interpretable representation of what the kernel is learning.
    \begin{figure}[!ht]
        \centering

        \includegraphics[width=0.40\textwidth]{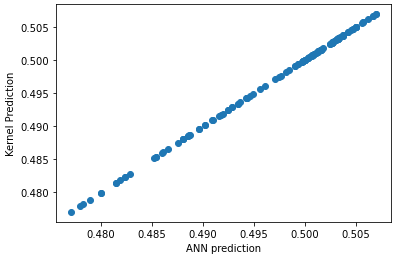}
        \caption{Class 1 EPK Kernel Prediction (Y) versus neural network prediction (X) for 100 test points, demonstrating extremely close agreement.}
        \label{fig:toymatch}
    \end{figure}

\subsection{Evaluating The Kernel} \label{subsec:evaluate}
A small test data set within 100 dimensions is created by generating 1000 random samples with means $(1,4,0,...)$, $(4,1,0,...)$ and $(5,5,0,...)$ and standard deviation $1.0$. These points are labeled according to the mean of the Gaussian used to generate them, providing 1000 points each from 3 classes. A fully connected ReLU network with 1 hidden layer is trained using categorical cross-entropy (CCE) and gradient descent with gradients aggregated across the entire training set for each step. We then compute the EPK for this network, approximating the integral from Equation~\ref{eq2} with 100 steps which replicates the output from the ReLU network within machine precision. The EPK (Kernel) outputs are compared with neural network predictions in Fig.~\ref{fig:toymatch} for class 1. Having established this kernel, and its corresponding kernel machine, one natural extension is to allow the kernel weights $a_i$ to be retrained. We perform this updating of the krenel weights using a SVM and present its predictions for each of three classes in Fig.~\ref{fig:svm}.

\subsection{Kernel Analysis}
Having established the efficacy of this kernel for model representation, the next step is to analyze this kernel to understand how it may inform us about the properties of the corresponding model. In practice, it becomes immediately apparent that this kernel lacks typical properties preferred when humans select kernels. Fig.~\ref{fig:svm} show that the weights of this kernel are non-stationary on our toy problem, with very stable model predictions far away from training data. Next, we use this kernel to estimate uncertainty. Consistent with many other research works on Gaussian processes for classification ~\cite{rasmussen2006gaussian} we use a GP to regress to logits. We then use Monte-Carlo to estimate posteriors with respect to probabilities (post-soft-max) for each prediction across a grid spanning the training points of our toy problem. The result is shown on the right-hand column of Fig.~\ref{fig:cov}. We can see that the kernel values are more confident (lower standard deviation) and more stable (higher kernel values) the farther they get from the training data in most directions. 

% show kernel values or mean prediction with variances from GP
    \begin{figure}[h]
        \centering
        \includegraphics[width=0.24\textwidth]{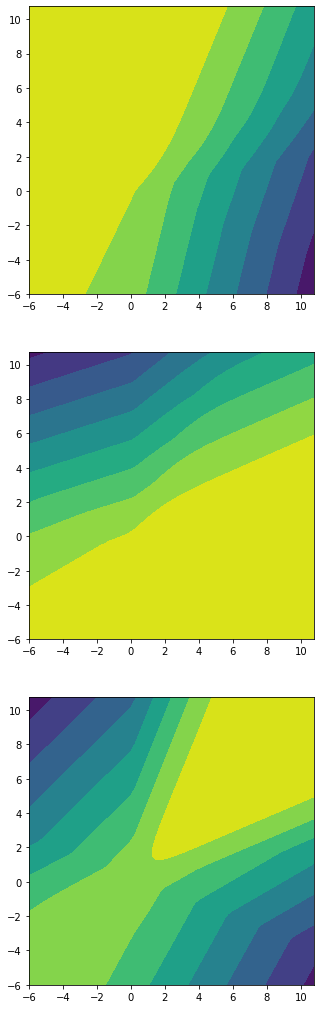}\includegraphics[width=0.24\textwidth]{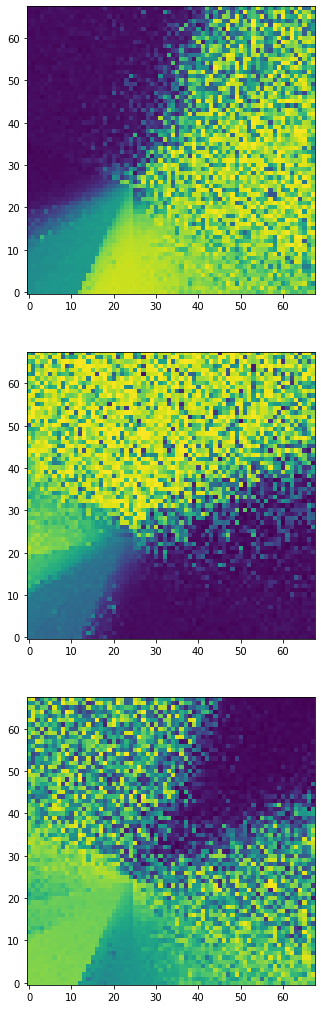}
        \caption{(left) Kernel values measured on a grid around the training set for our 2D problem. Bright yellow means high kernel value (right) Monte-Carlo estimated standard deviation based on gram matrices generated using our kernel for the same grid as the kernel values. Yellow means high standard deviation, blue means low standard deviation.}
        \label{fig:cov}
    \end{figure}

In order to further understand how these strange kernel properties come about, we exercise another advantage of a kernel by analyzing the points that are contributing to the kernel value for a variety of test points. 
In Fig.~\ref{fig:points} we examine the kernel values for each of the training points during evaluation of three points chosen as the mean of the generating distribution for each class. 
The most striking property of these kernel point values is the fact that they are not proportional to the euclidean distance from the test point.
This appears to indicate a set of basis vectors relative to each test point learned by the model based on the training data which are used to spatially transform the data in preparation for classification. This may relate to the correspondence between neural networks and maximum margin classifiers discussed in related work (~\cite{chizat2020maxmargin} ~\cite{shah2021input}). 
% In aggregate, the data are imposing a spatial transform on the test point and this transform is represented in the kernel weights by a smooth variation in the weights orthogonal to the basis function of this transform. 
% Our toy problem primarily varies in only 2 dimensions so these basis functions correspond with only normal vectors in 2 dimensions. 
Another more subtle property is that some individual data points, mostly close to decision boundaries are slightly over-weighted compared to the other points in their class. 
This latter property points to the fact that during the latter period of training, once the network has already achieved high accuracy, only the few points which continue to receive incorrect predictions, i.e. caught on the wrong side of a decision boundary, will continue contributing to the training gradient and therefore to the kernel value.

    \begin{figure*}[ht]
        \centering
        \includegraphics[width=0.32\textwidth]{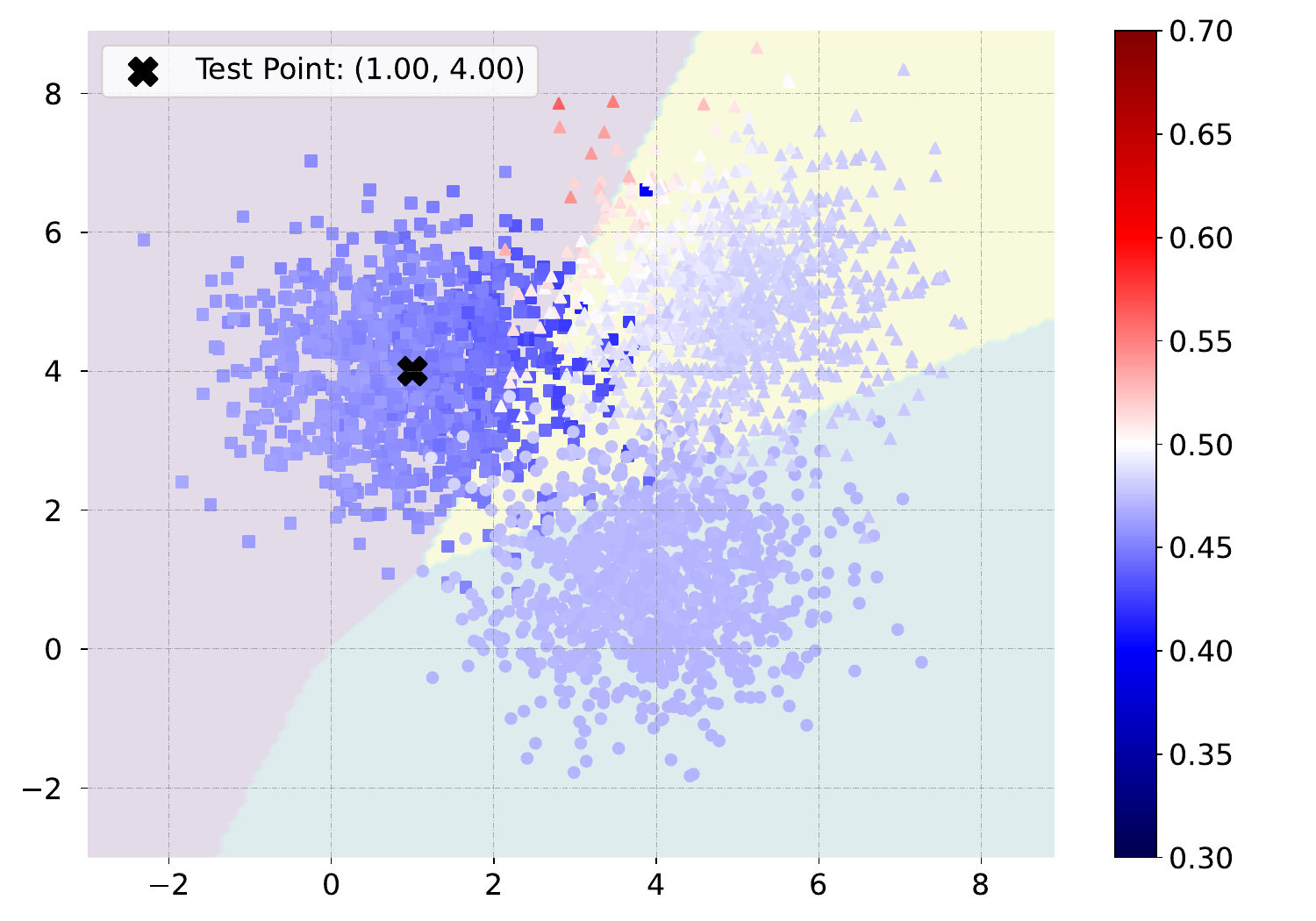}
        \includegraphics[width=0.32\textwidth]{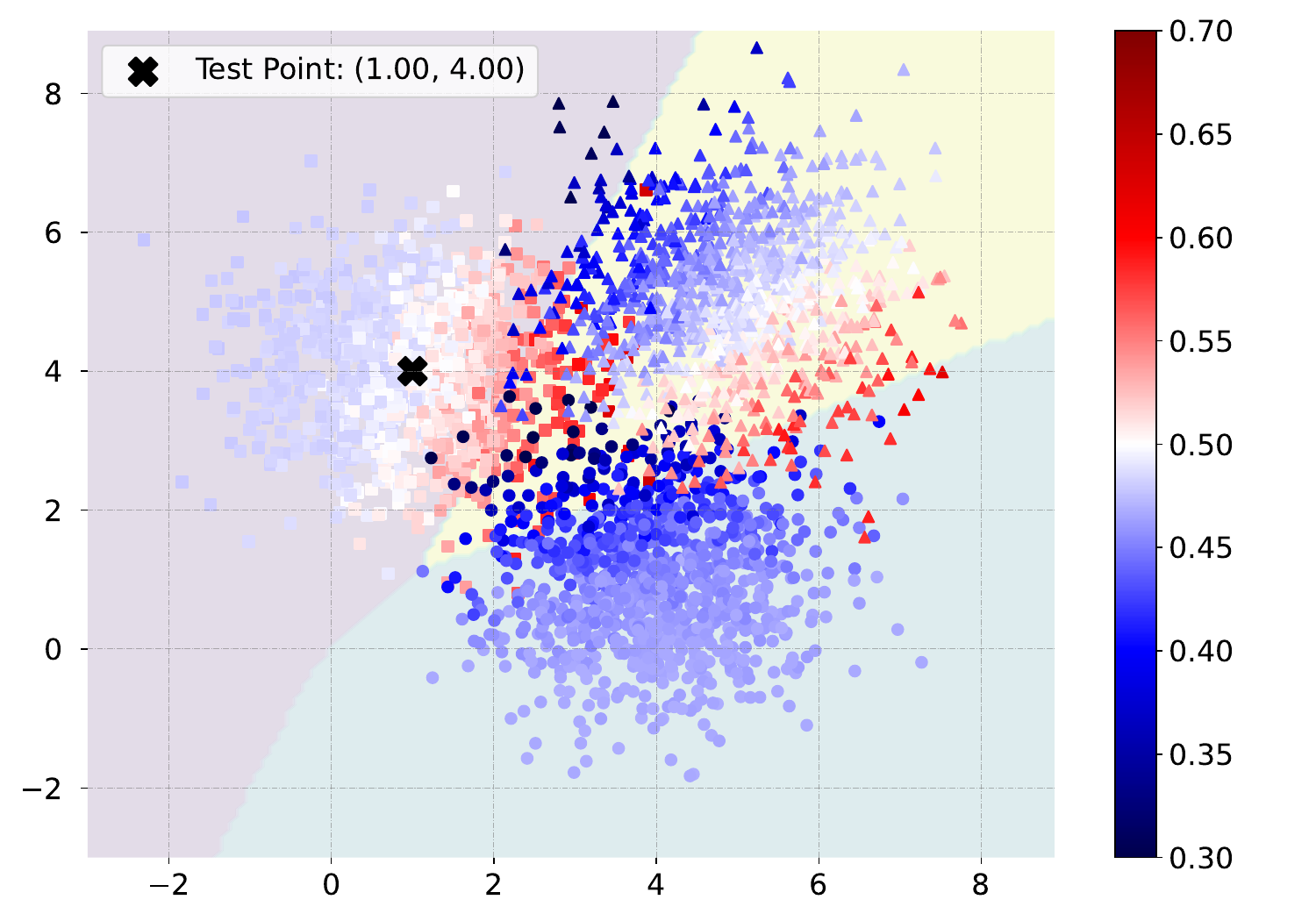}
        \includegraphics[width=0.32\textwidth]{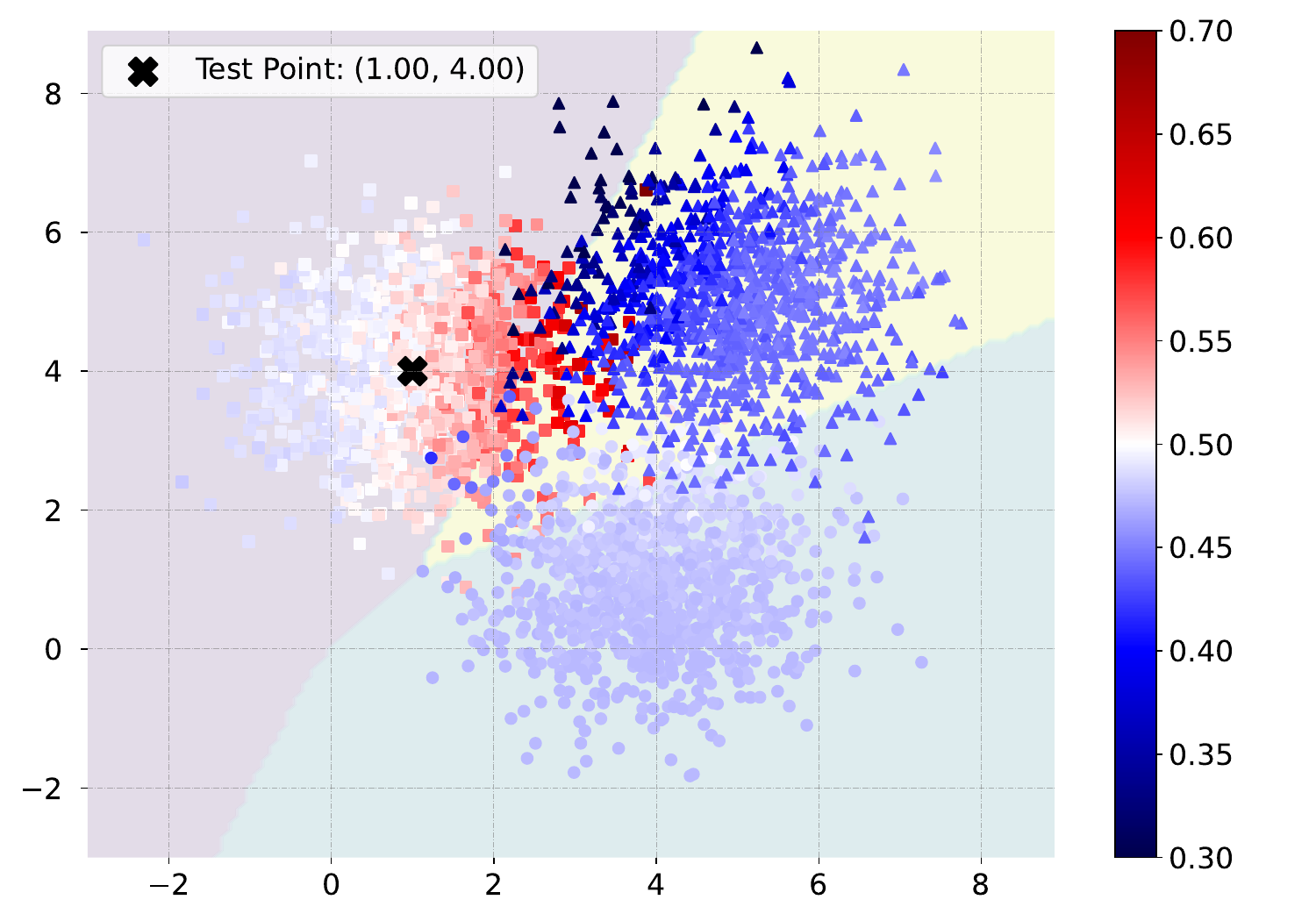}\\
        
        \includegraphics[width=0.32\textwidth]{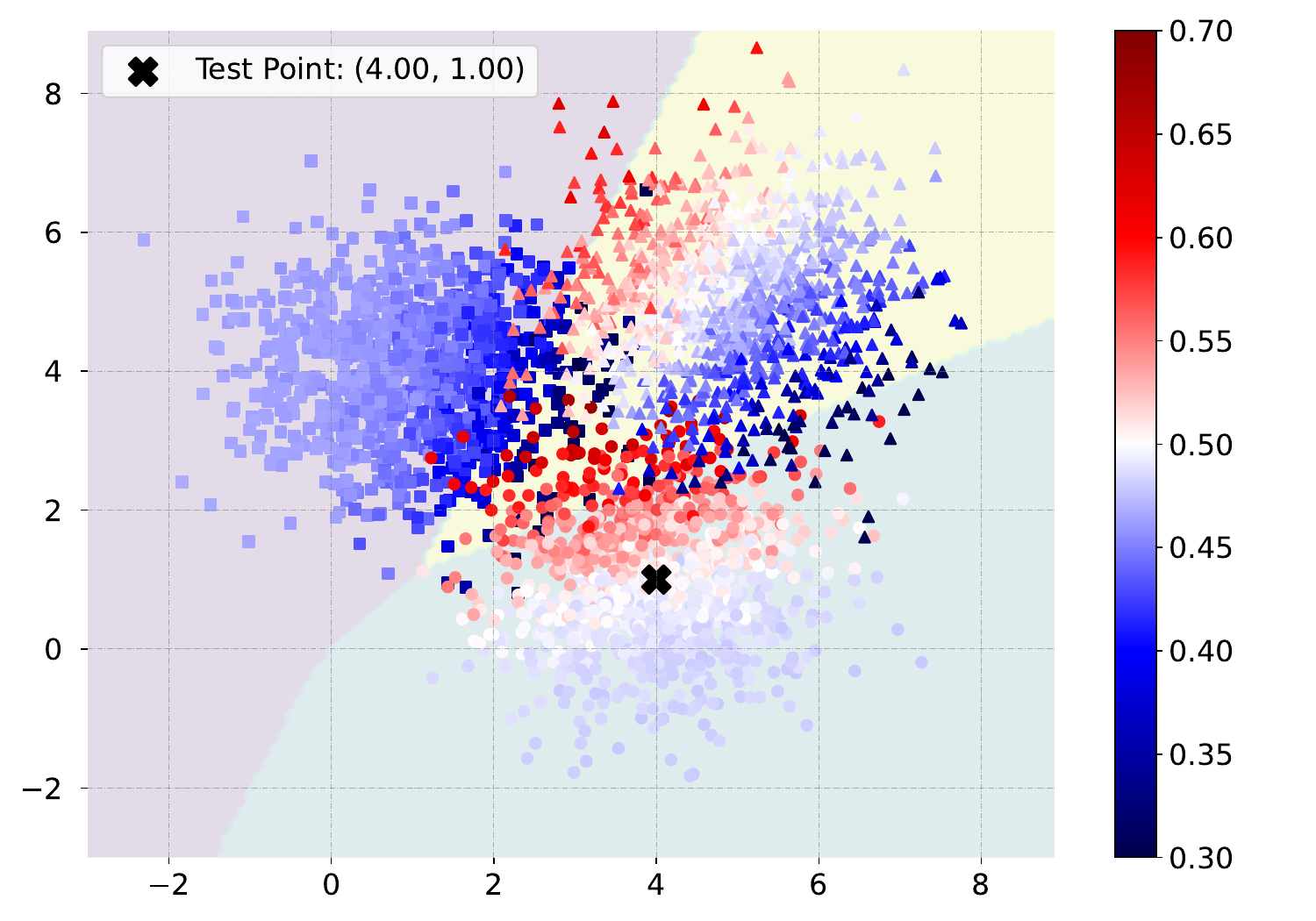}
        \includegraphics[width=0.32\textwidth]{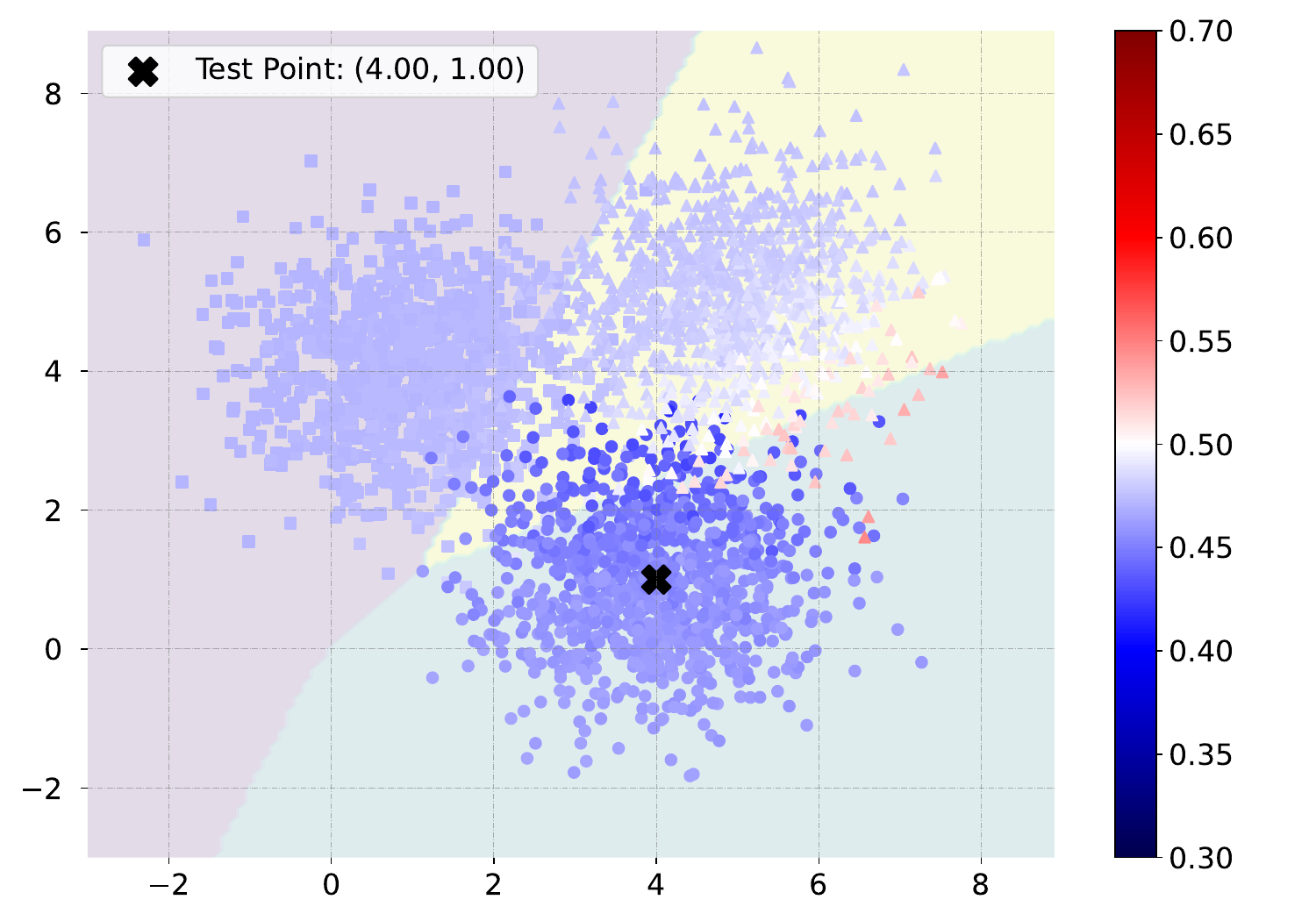}
        \includegraphics[width=0.32\textwidth]{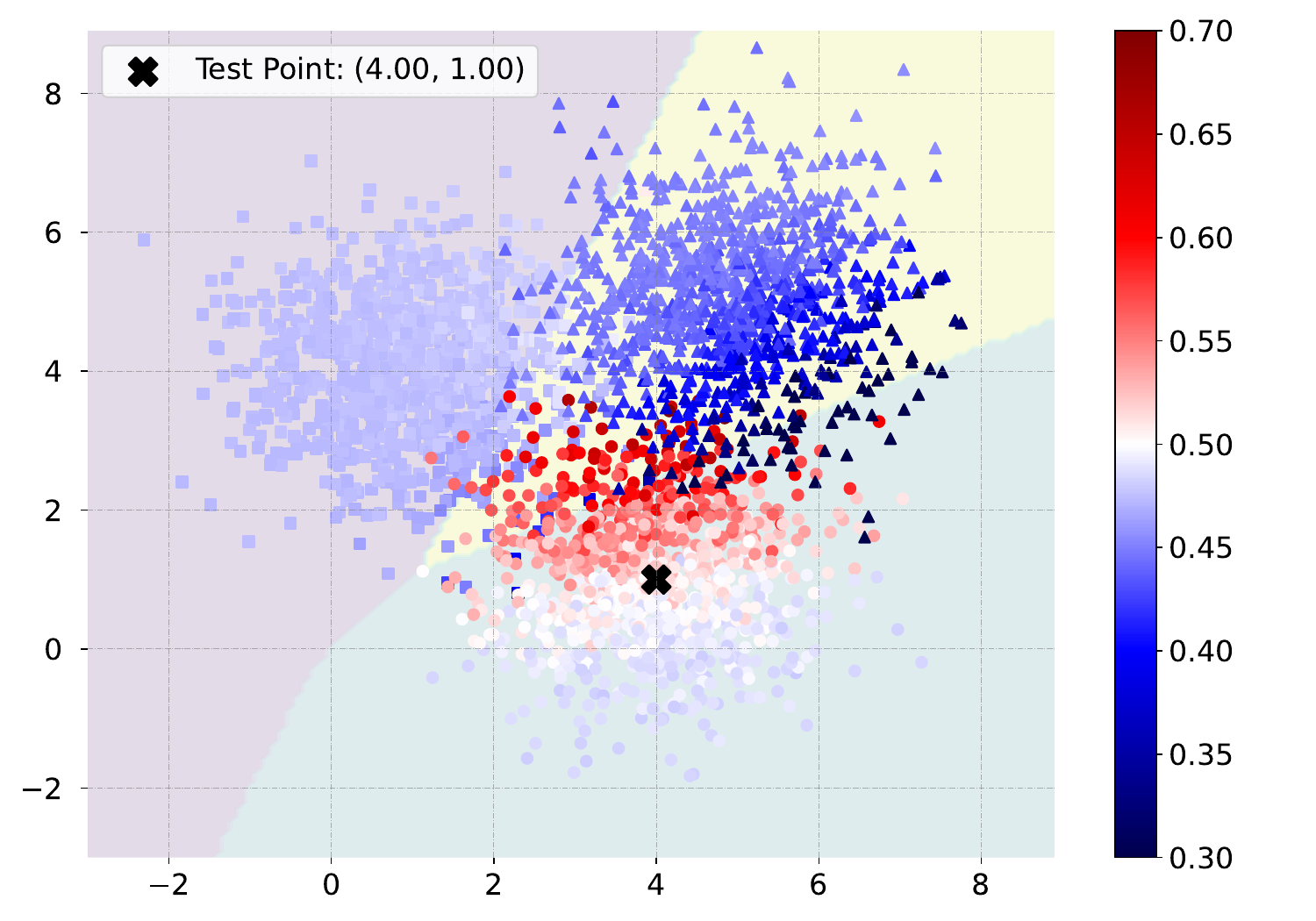}\\

        \includegraphics[width=0.32\textwidth]{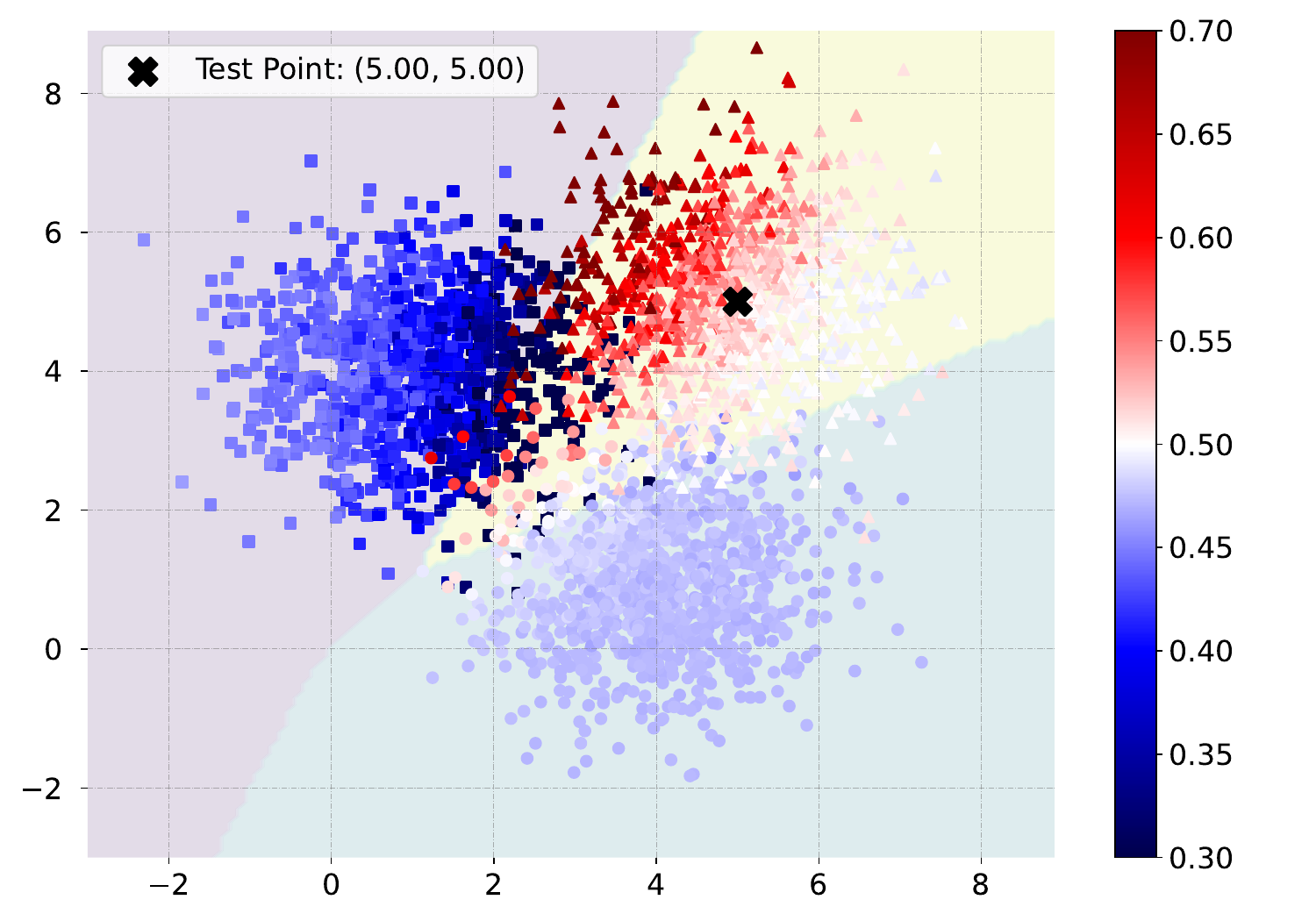}
        \includegraphics[width=0.32\textwidth]{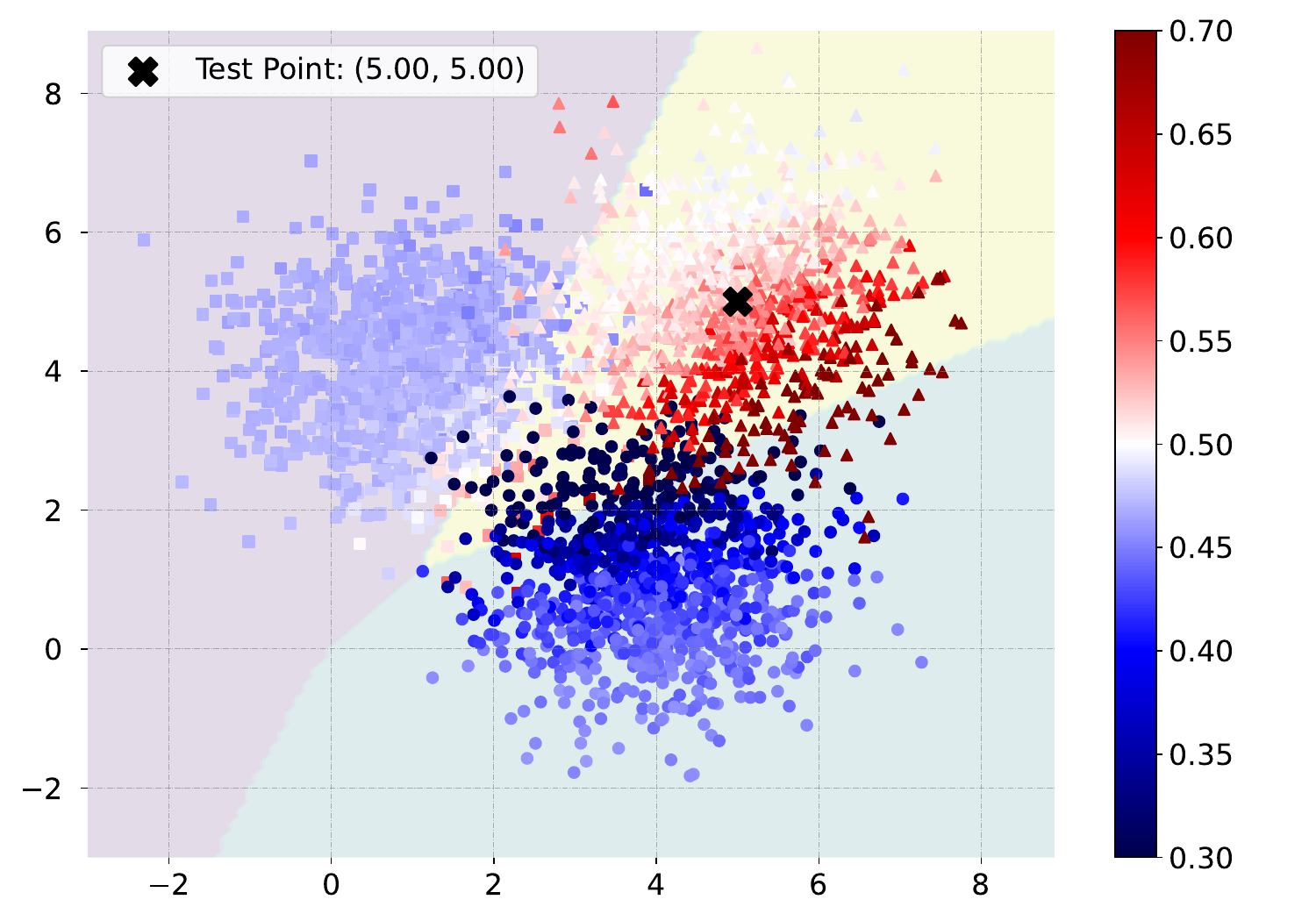}
        \includegraphics[width=0.32\textwidth]{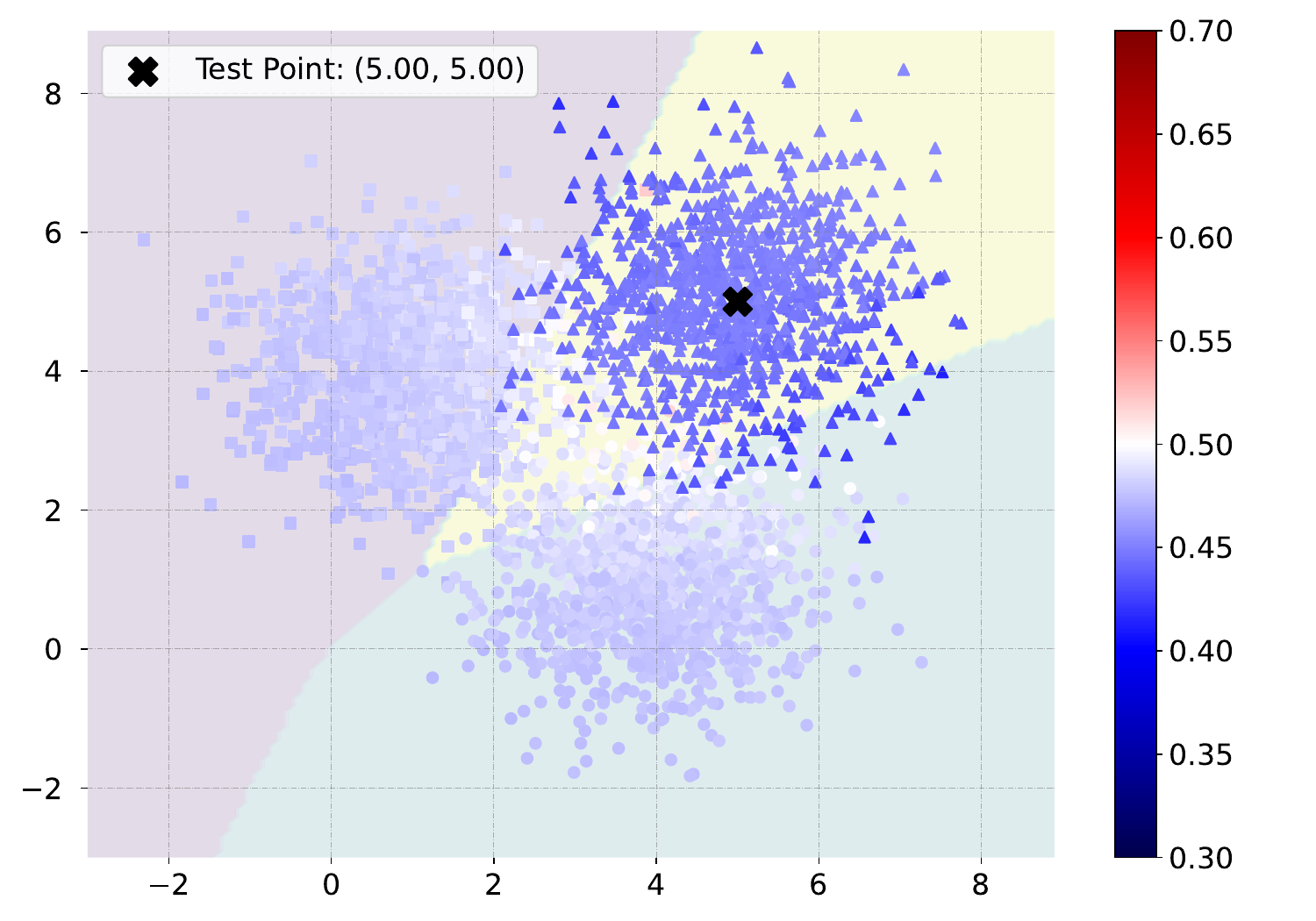}
        \caption{Plots showing kernel values for each training point relative to a test point. Because our kernel is replicating the output of a network, there are three kernel values per sample on a three class problem. This plot shows kernel values for all three classes across three different test points selected as the mean of the generating distribution. Figures on the diagonal show kernel values of the predicted class. Background shading is the neural network decision boundary.}
        \label{fig:points}
    \end{figure*}

% some more words

% \newpage

\subsection{Extending To Image Data}
    We perform experiments on MNIST to demonstrate the applicability to image data. 
    This kernel representation was generated for convolutional ReLU Network with the categorical cross-entropy loss function, using Pytorch ~\cite{pytorch}. 
    The model was trained using forward Euler (gradient descent) using gradients generated as a sum over all training data for each step. 
    The state of the model was saved for every training step. In order to compute the per-training-point gradients needed for the kernel representation, the per-input jacobians are computed at execution time in the representation by loading the model for each training step $i$, computing the jacobians for each training input to compute $\nabla_w f_{w_s(0)}(x_i)$, and then repeating this procedure for 200 $t$ values between 0 and 1 in order to approximate $\int_0^1 f_{w_s(t)}(x)$. For MNIST, the resulting prediction is very sensitive to the accuracy of this integral approximation, as shown in Fig.~\ref{fig:mnist}. The top plot shows approximation of the above integral with only one step, which corresponds to the DPK from previous work (~\cite{chen2021equivalence}, ~\cite{domingos2020}, ~\cite{incudini2022quantum}) and as we can see, careful approximation of this integral is necessary to achieve an accurate match between the model and kernel. 

\begin{figure}[!h]
    \centering
    % \vspace{-5mm}
    \includegraphics[width=0.95\linewidth]{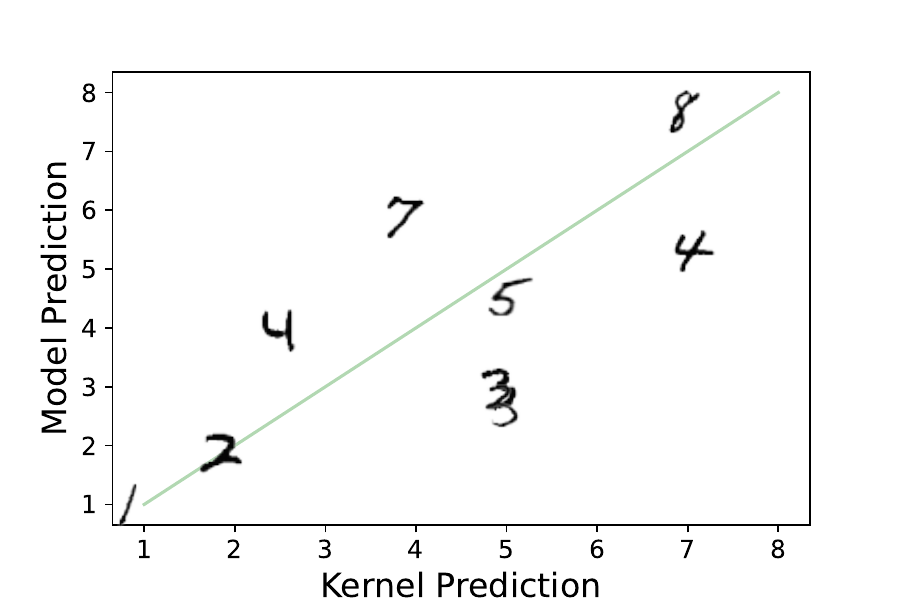}
    % \vspace{-18mm}

    \vspace{-9mm}
    
    \includegraphics[width=0.95\linewidth]{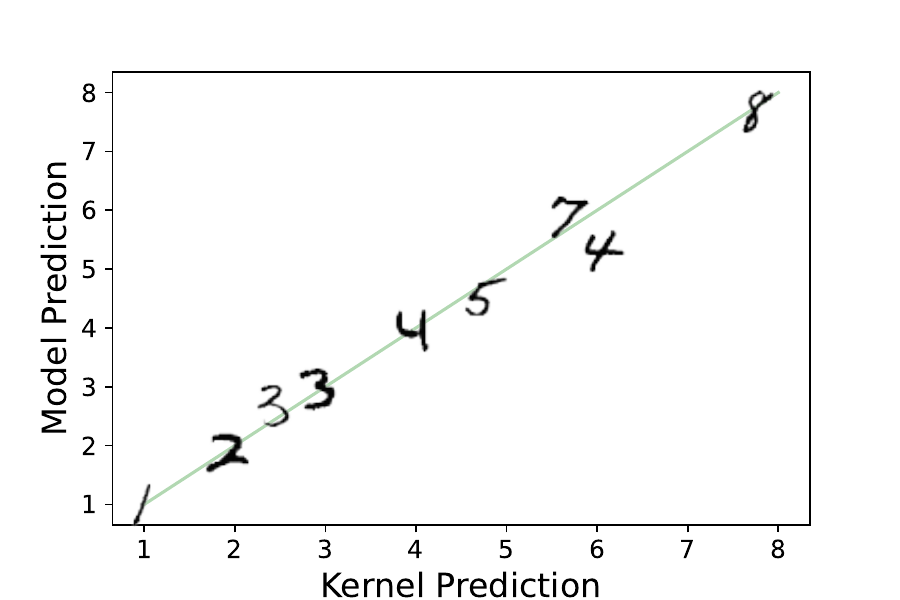}
    
    \vspace{-9mm}
    
    \includegraphics[width=0.95\linewidth]{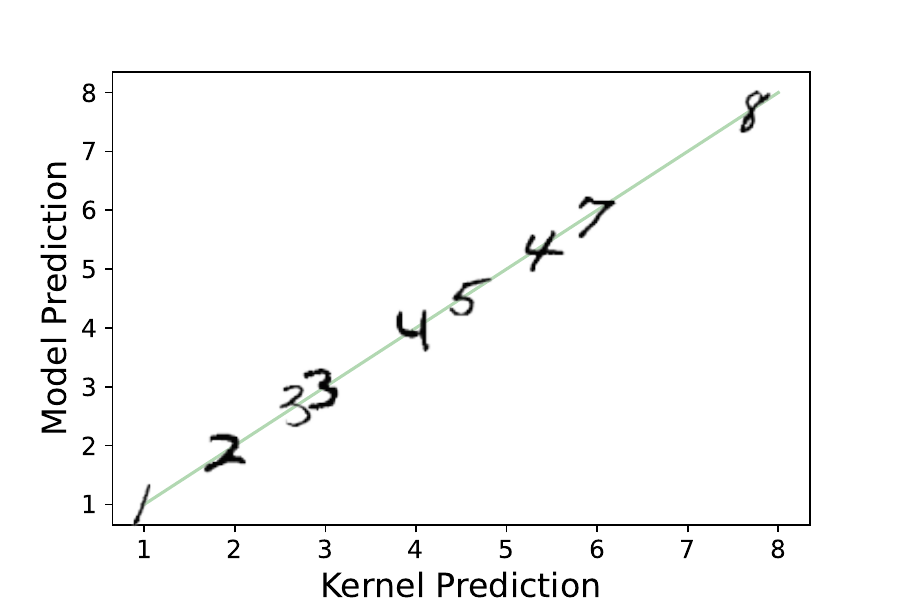}
    \caption{Experiment demonstrating the relationship between model predictions and kernel predictions for varying precision of the integrated path kernel. The top figure shows the integral estimated using only a single step. This is equivalent to the discrete path kernel (DPK) of previous work ~\cite{domingos2020every, chen2021equivalence}. The middle figure shows the kernel evaluated using 10 integral steps. The final figure shows the path kernel evaluated using 200 integral steps.}
    \label{fig:mnist}
\end{figure}
%I will leave it to michael to add any out-of-sample plots and plots related to $a_{i,s}$ weights and such. 
\section{Conclusion and Outlook} % confusion and onset
The implications of a practical and finite kernel representation for the study of neural networks are profound and yet importantly limited by the networks that they are built from. For most gradient trained models, there is a disconnect between the input space (e.g. images) and the parameter space of a network. Parameters are intrinsically difficult to interpret and much work has been spent building approximate mappings that convert model understanding back into the input space in order to interpret features, sample importance, and other details ~\cite{simonyan2013deep, lundberg2017unified, Selvaraju_2019}. The EPK is composed of a direct mapping from the input space into parameter space. This mapping allows for a much deeper understanding of gradient trained models because the internal state of the method has an exact representation mapped from the input space. As we have shown in Fig.~\ref{fig:points}, kernel values derived from gradient methods tell an odd story. We have observed a kernel that picks inputs near decision boundaries to emphasize and derives a spatial transform whose basis vectors depend neither uniformly nor continuously on training points. Although kernel values are linked to sample importance, we have shown that most contributions to the kernel's prediction for a given point are measuring an overall change in the network's internal representation. This supports the notion that most of what a network is doing is fitting a spatial transform based on a wide aggregation of data, and only doing a trivial calculation to the data once this spatial transform has been determined ~\cite{chizat2020maxmargin}. 
As stated in previous work ~\cite{domingos2020}, this representation has strong implications about the structure of gradient trained models and how they can understand the problems that they solve. Since the kernel weights in this representation are fixed derivatives with respect to the loss function $L$, $a_{i, s} = -\varepsilon  \dfrac{\partial L(f_{w_s(0)}(x_i),  y_i)}{\partial f_i}$, nearly all of the information used by the network is represented by the kernel mapping function and inner product. Inner products are not just measures of distance, they also measure angle. In fact, figure \ref{fig:grad} shows that for a typical training example, the $L_2$ norm of the weights changes monotonically by only 20-30\% during training. This means that the "learning" of a gradient trained model is dominated by change in angle, which is predicted for kernel methods in high dimensions ~\cite{hardle2004nonparametric}.

\begin{figure}[h]
\centering
\includegraphics[width=.45\textwidth]{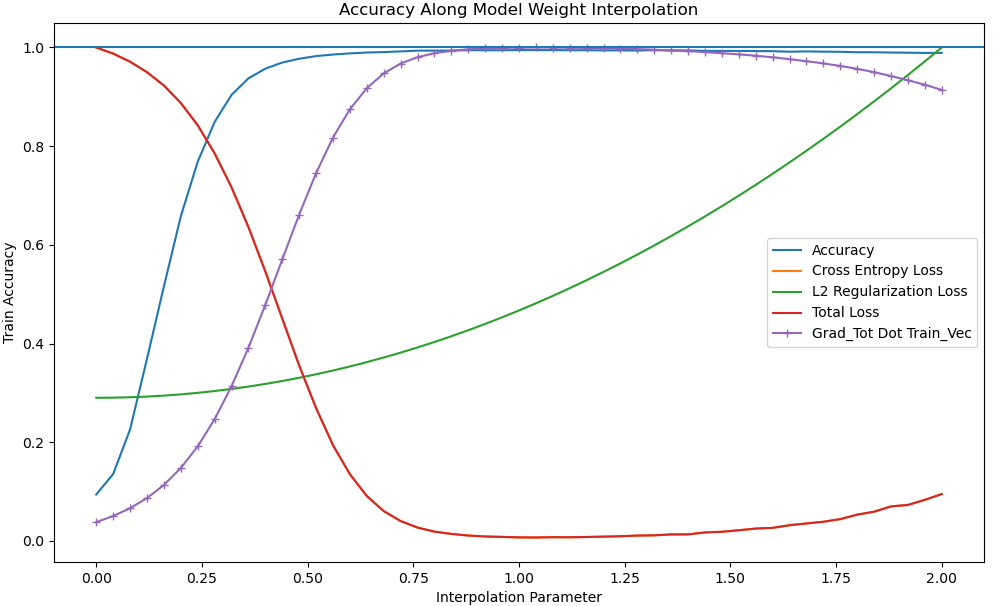}
\caption{This plot shows a linear interpolation $w(t) = w_0 + t(w_{1} - w_0)$ of model parameters $w$ for a convolutional neural network $f_w$ from their starting random state $w_0$ to their ending trained state $w_1$. The hatched purple line shows the dot product of the sum of the  gradient over the training data $X$, $\langle \nabla_w f_{w(t)}(X), (w_1 - w_0)/|w_1 - w_0| \rangle$. The other lines indicate accuracy (blue), total loss (red decreasing), and L2 Regularization (green increasing)}  
\label{fig:grad}
\end{figure}

% Perhaps the most significant advantage for gradient trained models of an exact kernel representation is that the combination of kernel and kernel weights provides a spatial representation of the model's understanding relative to the training data. In previous work (~\cite{gillette2022data} ~\cite{yousefzadeh2021deep} it has been shown that image classification can be represented by projection onto the convex hull of training data. This projection is computationally infeasible, but it provides a geometric gold-standard classifier in the native image space. Recent work ~\cite{chizat2020maxmargin} indicates that neural networks are in fact max margin classifiers in a metric space defined by their approximation of the wasserstein metric. Since kernel methods provide a spatial representation of their prediction which can be directly compared to the spatial classifier in image space, it can be used to analyze properties of the spatial transform that converts the computationally intractable convex hull in image space, to the computationally tractable approximated Wasserstein metric space. We can see the consequence of this in Fig.~\ref{fig:points}. 

For kernel methods, our result also represents a new direction. Despite their firm mathematical foundations, kernel methods have lost ground since the early 2000s because the features implicitly learned by deep neural networks yield better accuracy than any known hand-crafted kernels for complex high-dimensional problems ~\cite{NIPS2005_663772ea}. 
We're hopeful about the scalability of learned kernels based on recent results in scaling kernel methods  ~\cite{snelson2005sparse}. 
Exact kernel equivalence could allow the use of neural networks to implicitly construct a kernel. 
This could allow kernel based classifiers to approach the performance of neural networks on complex data. 
Kernels built in this way may be used with Gaussian processes to allow meaningful direct uncertainty measurement. 
This would allow for much more significant analysis for out-of-distribution samples including adversarial attacks ~\cite{szegedy2013intriguing, ilyas2019adversarial}. 
There is significant work to be done in improving the properties of the kernels learned by neural networks for these tools to be used in practice.
We are confident that this direct connection between practical neural networks and kernels is a strong first step towards achieving this goal.

\bibliography{bibfile}
\bibliographystyle{icml2023}

\appendix

\onecolumn
\section{Appendix}

\subsection{The EPK is a Kernel}

\ker*
\begin{proof}
We must show that the associated kernel matrix $K_{\text{EPK}} \in \mathbb{R}^{n\times n}$ defined for an arbitrary subset of data $\{x_i\}_{i=1}^M \subset X$ as $K_{\text{EPK},i,j} = \int_0^1\langle \phi_{s,t}(x_i), \phi_{s,t}(x_j)\rangle dt$ is both symmetric and positive semi-definite.

Since the inner product on a Hilbert space $\langle \cdot, \cdot \rangle$ is symmetric and since the same mapping $\varphi$ is used on the left and right, $K_{\text{EPK}}$ is \textbf{symmetric}. 

To see that $K_{\text{EPK}}$ is \textbf{Positive Semi-Definite}, let $\alpha = (\alpha_1, \alpha_2, \dots, \alpha_n)^\top \in \mathbb{R}^n$ be any vector. We need to show that $\alpha^\top K_{\text{EPK}} \alpha \geq 0$. We have

\begin{align}
\alpha^\top K_{\text{EPK}} \alpha &= \sum_{i=1}^n \sum_{j=1}^n \alpha_i \alpha_j \int_0^1 \langle \phi_{s,t}(x_i), \phi_{s,t}(x_j)\rangle dt \\
&= \sum_{i=1}^n \sum_{j=1}^n \alpha_i \alpha_j \int_0^1 \langle \nabla_{w}\hat{y}_{w_s(t,x_i)}, \nabla_{w}\hat{y}_{w_s(t,x_j)}\rangle dt \\
&= \int_0^1 \sum_{i=1}^n \sum_{j=1}^n \alpha_i \alpha_j \langle \nabla_{w}\hat{y}_{w_s(t,x_i)}, \nabla_{w}\hat{y}_{w_s(t,x_j)}\rangle dt \\
&= \int_0^1 \sum_{i=1}^n \sum_{j=1}^n  \langle \alpha_i \nabla_{w}\hat{y}_{w_s(t,x_i)}, \alpha_j \nabla_{w}\hat{y}_{w_s(t,x_j)}\rangle dt \\
&= \int_0^1    \langle \sum_{i=1}^n \alpha_i \nabla_{w}\hat{y}_{w_s(t,x_i)}, \sum_{j=1}^n \alpha_j \nabla_{w}\hat{y}_{w_s(t,x_j)}\rangle dt \\
& \text{Re-ordering the sums so that their indices match, we have}\\
&= \int_0^1 \left\lVert \sum_{i=1}^n \alpha_i \nabla_{w}\hat{y}_{w_s(t,x_i)}\right\rVert^2 dt \\
&\geq 0,
\end{align}

Note that this reordering does not depend on the continuity of our mapping function $\phi_{s,t}(x_i)$.

\end{proof}

\textbf{Remark \remlabel{rem2}} In the case that our mapping function $\varphi$ is not symmetric, after re-ordering, we still yield something of the form:
\begin{align}
&= \int_0^1 \left\lVert \sum_{i=1}^n \alpha_i \nabla_{w}\hat{y}_{w_s(t,x_i)}\right\rVert^2 dt \\
\end{align}
The natural asymmetric $\phi$ is symmetric for every non-training point, so we can partition this sum. For the non-training points, we have symmetry, so for those points we yield exactly the $L^2$ metric. For the remaining points, if we can pick a Lipschitz constant $E$ along the entire gradient field, then if training steps are enough, then the integral and the discrete step side of the asymmetric kernel will necessarily have positive inner product. In practice, this Lipschitz constant will change during training and for appropriately chosen step size (smaller early in training, larger later in training) we can guarantee positive-definiteness. In particular this only needs to be checked for training points. 

\subsection{The EPK gives an Exact Representation}
\label{proof:eker}
\eker*
\begin{proof}

Let $f_{w}$ be a differentiable function parameterized by parameters $w$ which is trained via $N$ forward Euler steps of fixed step size $\varepsilon$ on a training dataset $X$ with labels $ Y$, with initial parameters $w_0$ so that there is a constant $b$ such that for every $x$, $f_{w_0}(x) = b$, and weights at each step ${w_s : 0 \leq s \leq N}$. Let $x \in X$ be arbitrary and within the domain of $f_w$ for every $w$. For the final trained state of this model $f_{w_N}$, let $y = f_{w_N}(x)$. 

For one step of training, we consider $y_s  = f_{w_s(0)}(x)$ and $y_{s+1} = f_{w_{s+1}}(x)$. We wish to account for the change $y_{s+1} - y_s$ in terms of a gradient flow, so we must compute $\dfrac{\partial y}{dt}$ for a continuously varying parameter $t$. Since $f$ is trained using forward Euler with a step size of $\varepsilon > 0$, this derivative is determined by a step of fixed size of the weights $w_s$ to $w_{s+1}$. We parameterize this step in terms of the weights:

\begin{align}
    \dfrac{d w_s(t)}{dt} &= (w_{s+1} - w_s)\\   
    \int \dfrac{d w_s(t)}{dt} dt &= \int (w_{s+1} - w_s)dt\\
    w_s(t) &= w_s + t(w_{s+1} - w_s)\\
\end{align}
Since $f$ is being trained using forward Euler, across the entire training set $X$ we can write:
\begin{align}
    \dfrac{d w_s(t)}{dt} &= -\varepsilon \nabla_w L(f_{w_s(0)}(X), y_i) = -\varepsilon \sum_{j = 1}^{d} \sum_{i=1}^M  \dfrac{\partial L(f_{w_s(0)}(x_i),  y_i)}{\partial w_j} \label{eq10}
\end{align}
Applying chain rule and the above substitution, we can write
\begin{align}
    \dfrac{d \hat y}{dt} = \dfrac{d f_{w_s(t)}}{dt} &= \sum_{j = 1}^{d} \dfrac{d f}{\partial w_j} \dfrac{\partial w_j}{dt}\\
&= \sum_{j = 1}^{d} \dfrac{d f_{w_s(t)}(x)}{\partial w_j} \left(-\varepsilon \dfrac{\partial L(f_{w_s(0)}(X_T),  Y_T)}{\partial w_j}\right)\\
&= \sum_{j = 1}^{d} \dfrac{d f_{w_s(t)}(x)}{\partial w_j} \left(-\varepsilon \sum_{i = 1}^{M}\dfrac{d L(f_{w_s(0)}(x_i),  y_i)}{d f_{w_s(0)}(x_i)}\dfrac{\partial  f_{w_s(0)}(x_i)}{\partial w_j}\right)\\
&= -\varepsilon \sum_{i = 1}^{M} \dfrac{d L(f_{w_s(0)}(x_i),  y_i)}{d f_{w_s(0)}(x_i)} \sum_{j = 1}^{d} \dfrac{d f_{w_s(t)}(x)}{\partial w_j}  \dfrac{d f_{w_s(0)}(x_i)}{\partial w_j}\\
&= -\varepsilon \sum_{i = 1}^{M} \dfrac{d L(f_{w_s(0)}(x_i),  y_i)}{d f_{w_s(0)}(x_i)} \nabla_w f_{w_s(t)}(x) \cdot \nabla_w f_{w_s(0)}(x_i)\\
\end{align}
Using the fundamental theorem of calculus, we can compute the change in the model's output over step $s$
\begin{align}
    y_{s+1} - y_s &= \int_0^1 -\varepsilon \sum_{i = 1}^{M} \dfrac{d L(f_{w_s(0)}(x_i),  y_i)}{d f_{w_s(0)}(x_i)}  \nabla_w f_{w_s(t)}(x) \cdot \nabla_w f_{w_s(0)}(x_i)dt\\
 &=  -\varepsilon \sum_{i = 1}^{M} \dfrac{d L(f_{w_s(0)}(x_i),  y_i)}{d f_{w_s(0)}(x_i)}  \left(\int_0^1\nabla_w f_{w_s(t)}(x)dt\right) \cdot \nabla_w f_{w_s(0)}(x_i)\\
\end{align}
For all $N$ training steps, we have
\begin{align*}
y_N &= b + \sum_{s=1}^N y_{s+1} - y_s\\
y_N &= b + \sum_{s = 1}^N -\varepsilon \sum_{i = 1}^{M} \dfrac{d L(f_{w_s(0)}(x_i),  y_i)}{d f_{w_s(0)}(x_i)}  \left(\int_0^1\nabla_w f_{w_s(t)}(x)dt\right) \cdot \nabla_w f_{w_s(0)}(x_i)\\
% &= \sum_{i = 1}^{M}\sum_{s = 1}^N -\varepsilon  \dfrac{\partial L(f_{w_s(0)}(x_i),  y_i)}{\partial f_i}  \left(\int_0^1\nabla_w f_{w_s(t)}(x)dt\right) \cdot \nabla_w f_{w_s(0)}(x_i)\\
% &= \sum_{i = 1}^{M}\sum_{s = 1}^N -\varepsilon  \dfrac{\partial L(f_{w_s(0)}(x_i),  y_i)}{\partial f_i}  \int_0^1\left\langle \nabla_w f_{w_s(t)}(x), \nabla_w f_{w_s(0)}(x_i) \right\rangle dt\\ 
&= b + \sum_{i = 1}^{M}\sum_{s = 1}^N -\varepsilon  \dfrac{d L(f_{w_s(0)}(x_i),  y_i)}{d f_{w_s(0)}(x_i)}  \int_0^1\left\langle \nabla_w f_{w_s(t,x)}(x), \nabla_w f_{w_s(t,x_i)}(x_i) \right\rangle dt\\ 
&= b + \sum_{i = 1}^{M}\sum_{s = 1}^N a_{i, s}  \int_0^1 \left\langle \phi_{s,t}(x), \phi_{s,t}(x_i)\right\rangle dt
\end{align*}
Since an integral of a symmetric positive semi-definite function is still symmetric and positive-definite, each step is thus represented by a kernel machine. 

\end{proof}
\subsection{When is an Ensemble of Kernel Machines itself a Kernel Machine?}
\label{proof:ekmr}
Here we investigate when our derived ensemble of kernel machines composes to a single kernel machine.
In order to show that a linear combination of kernels also equates to a kernel it is sufficient to show that $a_{i,s} = a_{i,0}$ for all $a_{i,s}$.
The $a_{i}$ terms in our kernel machine are determined by the gradient the training loss function.
This statement then implies that the gradient of the loss term must be constant throughout training in order to form a kernel.
Here we show that this is the case when we consider a log softmax activation on the final layer and a negative log likelihood loss function.
% In this case it is possible to let the sample weights of our final kernel machine equal $a_{i,0}$.
% In order to show this, we impose some structure on the loss function and network.
% Here we show this is the case for binary crossentropy on a network with sigmoid activations on the logits.
% (TODO: More argument here using mercer's theorem. All positive linear combinations of kernels are kernels. There are cases where some negative coefficients are allowed but that's going to take a lot more thought. How do we extend this to say $aKa > 0$ for all $a$?)

\begin{proof}
Assume a two class problem. In the case of a function with multiple outputs, we consider each output to be a kernel. We define our network output $\hat y_i$ as all layers up to and including the log softmax and $y_i$ is a one-hot encoded vector. 

\begin{align}
    L(\hat y_i,  y_i)
    &=  \sum_{k=1}^K -y_i^k(\hat y_i^k) \\
    % \dfrac{\partial L(\hat y_i,  y_i)}{\partial \hat y_i} &= -y_i - (1-y_i)\\
    % &= -1
\end{align}
For a given output indexed by $k$, if $y_i^k = 1$ then we have
\begin{align}
    L(\hat y_i,  y_i)
    &=  -1(\hat y_i^k) \\
    \dfrac{\partial L(\hat y_i,  y_i)}{\partial \hat y_i} &= -1\\
\end{align}
If $y_i^k = 0$ then we have
\begin{align}
    L(\hat y_i,  y_i)
    &=  0(\hat y_i^k) \\
    \dfrac{\partial L(\hat y_i,  y_i)}{\partial \hat y_i} &= 0\\
\end{align}
% Here we can see that the gradient of the loss value is fixed for all weight states, therefore 

% For a binary classification problem it is standard to have $y_i \in \{0, 1\}$ and using a sigmoid activation on the final layer we have $f_i \in (0, 1)$. \\

% \begin{center}
    
% \begin{minipage}{0.45\textwidth}
% Assume $y_i = 0$.
% \begin{align}
%     \dfrac{\partial L(\hat y_i,  y_i)}{\partial \hat y_i} &= \dfrac{0 - \hat y_i}{(\hat y_i - 1) \hat y_i}\\
%     &= \dfrac{-1}{\hat y_i - 1}\\
%     &= \dfrac{1}{|\hat y_i - 1|}
% \end{align}
% The last equality relies on the fact that $\hat y_i < 1$.
% \begin{equation}
%     y_i = 0 \implies \dfrac{\partial L(\hat y_i,  y_i)}{\partial \hat y_i} > 0
% \end{equation}
% \end{minipage}
% \hspace{0.04\textwidth}
% \begin{minipage}{0.45\textwidth}
% Assume $y_i = 1$.
% \begin{align}
%     \dfrac{\partial L(\hat y_i,  y_i)}{\partial \hat y_i} &= \dfrac{1 - \hat y_i}{(\hat y_i - 1) \hat y_i}\\
%     &= \dfrac{1 - \hat y_i}{-(1-\hat y_i) \hat y_i}\\
%     &= -\dfrac{1}{\hat y_i}
% \end{align}
% Because $\hat y_i > 0$.
% \begin{equation}
%     y_i = 1 \implies \dfrac{\partial L(\hat y_i,  y_i)}{\partial \hat y_i} < 0
% \end{equation}
% \end{minipage}
% \end{center}

In this case, since the loss is scaled directly by the output, and the only other term is an indicator function deciding which class label to take, we get a constant gradient.
This shows that the gradient of the loss function does not depend on $\hat y_i$. 
Therefore:
\begin{align}
    y &= b - \varepsilon \sum_{i = 1}^{N}\sum_{s = 1}^S a_{i, s}  \int_0^1 \left\langle \phi_{s,t}(x), \phi_{s,t}(x_i)\right\rangle dt\\
     &= b - \varepsilon \sum_{i = 1}^{N} a_{i, 0} \sum_{s = 1}^S  \int_0^1 \left\langle \phi_{s,t}(x), \phi_{s,t}(x_i)\right\rangle dt
\end{align}
This formulates a kernel machine where
\begin{align}
a_{i, 0} &= \dfrac{\partial L(f_{w_0}(x_i),  y_i)}{\partial f_i} \\
K(x, x_i) &= \sum_{s = 1}^S \int_0^1 \left\langle \phi_{s,t}(x), \phi_{s,t}(x_i)\right\rangle dt \\
\phi_{s,t}(x) &=  \nabla_w f_{w_s(t,x)} (x)\\
w_s(t,x) &= \begin{cases} w_s, x \in X_T\\ w_s(t), x \notin X_T \end{cases}\\
b &= 0
\end{align}
\end{proof}

It is important to note that this does not hold up if we consider the log softmax function to be part of the loss instead of the network.
In addition, there are loss structures which can not be rearranged to allow this property.
In the simple case of linear regression, we can not disentangle the loss gradients from the kernel formulation, preventing the construction of a valid kernel. 
For example assume our loss is instead squared error. Our labels are continuous on $\mathds{R}$ and our activation is the identity function.
\begin{align}
    L(f_i,  y_i) 
    &= (y_i - f_{i, s})^2 \\
    \dfrac{\partial L(f_i,  y_i)}{\partial f_i} &= 2(y_i- f_{i, s})
\end{align}

This quantity is dependent on $f_i$ and its value is changing throughout training. %(TODO: Make this more formal and rigorous)

In order for 
\begin{align}
    \sum_{s=1}^S a_{i,s} \int_0^1 \langle \phi_{s,t}(x), \phi_{s,t}(x_i)\rangle dt
\end{align}
to be a kernel on its own, we need it to be a positive (or negative) definite operator and symetric. In the specific case of our practical path kernel, i.e. that in $K(x,x')$ if $x'$ happens to be equal to $x_i$, then positive semi-definiteness can be accounted for:
\begin{align}
    &= \sum_{s=1}^S 2(y_i- f_{i, s}) \int_0^1 \langle \phi_{s,t}(x), \phi_{s,t}(x_i)\rangle dt\\
    &= \sum_{s=1}^S 2(y_i- f_{i, s}) \int_0^1 \langle \nabla_w f_{w_s(t))} (x), \nabla_w f_{w_s(0)} (x_i)\rangle dt\\
    &= \sum_{s=1}^S 2 \left(y_i \cdot \int_0^1 \langle \nabla_w f_{w_s(t))} (x), \nabla_w f_{w_s(0)} (x_i)\rangle dt - f_{i, s} \int_0^1 \langle \nabla_w f_{w_s(t))} (x), \nabla_w f_{w_s(0)} (x_i)\rangle dt \right)\\
    &= \sum_{s=1}^S 2 \left(y_i \cdot \int_0^1 \langle \nabla_w f_{w_s(t))} (x), \nabla_w f_{w_s(0)} (x_i)\rangle dt -  \int_0^1 \langle \nabla_w f_{w_s(t))} (x), f_{i, s} \nabla_w f_{w_s(0)} (x_i)\rangle dt \right)\\
    &= \sum_{s=1}^S 2 \left(y_i \cdot \int_0^1 \langle \nabla_w f_{w_s(t))} (x), \nabla_w f_{w_s(0)} (x_i)\rangle dt -  \int_0^1 \langle \nabla_w f_{w_s(t))} (x), \dfrac{1}{2}\nabla_w (f_{w_s(0)} (x_i))^2\rangle dt \right)\\
\end{align}
Otherwise, we get the usual 
\begin{align}
        &= \sum_{s=1}^S 2(y_i- f_{i, s}) \int_0^1 \langle \nabla_w f_{w_s(t,x))} (x), \nabla_w f_{w_s(t,x)} (x')\rangle dt\\
\end{align}
The question is two fold. One, in general theory (i.e. the lower example), can we contrive two pairs $(x_1,x'_1)$ and $(x_2,x'_2)$ that don't necessarily need to be training or test images for which this sum is positive for $1$ and negative for $2$. Second, in the case that we are always comparing against training images, do we get something more predictable since there is greater dependence on $x_i$ and we get the above way of re-writing  using the gradient of the square of $f(x_i)$. 

However, even accounting for this by removing the sign of the loss will still produce a non-symmetric function. This limitation is more difficult to overcome.

\subsection{Multi-Class Case}

There are two ways of treating our loss function $L$ for a number of classes (or number of output activations) $K$:
\begin{align}
    \text{Case 1: } L &: \mathbb{R}^K \to \mathbb{R}\\
    \text{Case 2: } L &: \mathbb{R}^K \to \mathbb{R}^K\\
\end{align}

\subsubsection{Case 1 Scalar Loss}

Let $L : \mathbb{R}^K \to \mathbb{R}$. We use the chain rule $D (g \circ f) (x) = Dg(f(x))Df(x)$. 

Let $f$ be a vector valued function so that $f : \mathbb{R}^D \to \mathbb{R}^K$  satisfying the conditions from [representation theorem above] with $x \in \mathbb{R}^D$ and $y_i \in \mathbb{R}^K$ for every $i$. We note that $\dfrac{\partial f}{\partial t}$ is a column and has shape $Kx1$ and our first chain rule can be done the old fashioned way on each row of that column:
\begin{align}
    \dfrac{d f}{d t} &= \sum_{j=1}^M \dfrac{d f(x)}{\partial w_j} \dfrac{d w_j}{d t}\\
    &= -\varepsilon \sum_{j=1}^M \dfrac{d f(x)}{\partial w_j} \sum_{i=1}^N \dfrac{\partial L(f(x_i), y_i)}{\partial w_j}\\
    &\text{Apply chain rule}\\
    &= -\varepsilon \sum_{j=1}^M \dfrac{d f(x)}{\partial w_j} \sum_{i=1}^N \dfrac{\partial L(f(x_i), y_i)}{\partial f}\dfrac{d f(x_i)}{\partial w_j}\\
    &\text{Let}\\
    A &= \dfrac{d f(x)}{\partial w_j} \in \mathbb{R}^{K \times 1}\\
    B &= \dfrac{d L(f(x_i), y_i)}{d f} \in \mathbb{R}^{1 \times K}\\
    C &= \dfrac{d f(x_i)}{\partial w_j} \in \mathbb{R}^{K \times 1}
\end{align}
We have a matrix multiplication $ABC$ and we wish to swap the order so somehow we can pull $B$ out, leaving $A$ and $C$ to compose our product for the representation. Since $BC \in \mathbb{R}$, we have $(BC) = (BC)^T$ and we can write
\begin{align}
    (ABC)^T &= (BC)^TA^T = BCA^T\\
    ABC &= (BCA^T)^T
\end{align}
Note: This condition needs to be checked carefully for other formulations so that we can re-order the product as follows:
\begin{align}
        &= -\varepsilon \sum_{j=1}^M  \sum_{i=1}^N \left(\dfrac{d L(f(x_i), y_i)}{d f} 
        \dfrac{d f(x_i)}{\partial w_j} \left(\dfrac{d f(x)}{\partial  w_j}\right)^T\right)^T
        \\
    &= -\varepsilon \sum_{i=1}^N \left(\dfrac{d L(f(x_i), y_i)}{d f} 
    \sum_{j=1}^M \dfrac{d f(x_i)}{\partial w_j} \left(\dfrac{d f(x)}{\partial w_j}\right)^T\right)^T\\        
\end{align}
Note, now that we are summing over $j$, so we can write this as an inner product on $j$ with the $\nabla$ operator which in this case is computing the jacobian of $f$ along the dimensions of class (index k) and weight (index j). We can define 
\begin{align}
    (\nabla f(x))_{k,j} &= \dfrac{d f_{k}(x)}{\partial w_j}\\
    &= -\varepsilon \sum_{i=1}^N \left(\dfrac{d L(f(x_i), y_i)}{d f} 
     \nabla f(x_i) (\nabla f(x))^T\right)^T\\    
\end{align}
We note that the dimensions of each of these matrices in order are $[1,K]$, $[K,M]$, and $[M,K]$ which will yield a matrix of dimension $[1, K]$ i.e. a row vector which we then transpose to get back a column of shape $[K, 1]$. Also, we note that our kernel inner product now has shape $[K,K]$. 

\subsection{Schemes Other than Forward Euler (SGD)}

\textbf{Variable Step Size:}
Suppose $f$ is being trained using Variable step sizes so that across the training set $X$:
\begin{align}
    \dfrac{d w_s(t)}{dt} &= -\varepsilon_s \nabla_w L(f_{w_s(0)}(X), y_i) = -\varepsilon \sum_{j = 1}^{d} \sum_{i=1}^M  \dfrac{\partial L(f_{w_s(0)}(X),  y_i)}{\partial w_j} \label{eq10}
\end{align}
This additional dependence of $\varepsilon$ on $s$ simply forces us to keep $\varepsilon$ inside the summation in equation~\ref{eq11}. 

\textbf{Other  Numerical Schemes:} Suppose $f$ is being trained using another numerical scheme so that:
\begin{align}
    \dfrac{d w_s(t)}{dt} &= \varepsilon_{s,l} \nabla_w L(f_{w_s(0)}(x_i), y_i) + \varepsilon_{s-1, l}\nabla_w L(f_{w_{s-1}}(x_i), y_i) + \cdots \\
    &= \varepsilon_{s,l} \sum_{j = 1}^{d} \sum_{i=1}^M  \dfrac{\partial L(f_{w_s(0)}(x_i),  y_i)}{\partial w_j} + \varepsilon_{s-1, l} \sum_{j = 1}^{d} \sum_{i=1}^M  \dfrac{\partial L(f_{w_{s-1}(0)}(x_i),  y_i)}{\partial w_j} + \cdots
\end{align}
This additional dependence of $\varepsilon$ on $s$ and $l$ simply results in an additional summation in equation~\ref{eq11}. Since addition commutes through kernels, this allows separation into a separate kernel for each step contribution. Leapfrog and other first order schemes will fit this category. 

\textbf{Higher Order Schemes:} Luckily these are intractable for for most machine-learning models because they would require introducing dependence of the kernel on input data or require drastic changes. It is an open but intractable problem to derive kernels corresponding to higher order methods.

\subsection{Variance Estimation}
In order to estimate variance we treat our derived kernel function $K$ as the covariance function for Gaussian process regression. Given training data $X$ and test data $X'$, we can use the Kriging to write the mean prediction and it variance from true $\mu(x)$ as
\begin{align}
    \bar \mu(X') &= \left[ K(X', X)\right] \left[ K(X, X) \right]^{-1} \left[ Y \right]\\
    \text{Var}(\bar \mu(X') - \mu(X')) &= \left[K(X', X') \right] - \left[ K(X', X)\right] \left[ K(X, X) \right]^{-1} \left[ K(X, X') \right]\\
\end{align}
Where 
\begin{align}
    \left[K(A, B) \right] &= \begin{bmatrix} K(A_1, B_1) & K(A_1, B_2) & \cdots \\
                                             K(A_2, B_1) & K(A_2, B_2) &  \\
                                             \vdots &  & \ddots \\ \end{bmatrix}
\end{align}
To finish our Gaussian estimation, we note that each $K(A_i, B_i)$ will me a $k$ by $k$ matrix where $k$ is the number of classes. We take $\text{tr}(K(A_i, B_i)$ to determine total variance for each prediction. 

\section*{Acknowledgements}

This material is based upon work supported by the Department of Energy (National Nuclear Security Administration Minority Serving Institution Partnership Program's CONNECT - the COnsortium on Nuclear sECurity Technologies) DE-NA0004107.
This report was prepared as an account of work sponsored by an agency of the United States Government.
Neither the United States Government nor any agency thereof, nor any of their employees, makes any warranty, express or implied, or assumes any legal liability or responsibility for the accuracy, completeness, or usefulness of any information, apparatus, product, or process disclosed, or represents that its use would not infringe privately owned rights. The views and opinions of authors expressed herein do not necessarily state or reflect those of the United States Government or any agency thereof.

This material is based upon work supported by the National Science Foundation under Grant No. 2134237. We would like to additionally acknowledge funding from NSF TRIPODS Award Number 1740858 and NSF RTG Applied Mathematics and Statistics for Data-Driven Discovery Award Number 1937229. Any opinions, findings, and conclusions or recommendations expressed in this material are those of the author(s) and do not necessarily reflect the views of the National Science Foundation.
% \textbf{Do not} include acknowledgements in the initial version of
% the paper submitted for blind review.

% If a paper is accepted, the final camera-ready version can (and
% probably should) include acknowledgements. In this case, please
% place such acknowledgements in an unnumbered section at the
% end of the paper. Typically, this will include thanks to reviewers
% who gave useful comments, to colleagues who contributed to the ideas,
% and to funding agencies and corporate sponsors that provided financial

% % In the unusual situation where you want a paper to appear in the
% % references without citing it in the main text, use \nocite

% %%%%%%%%%%%%%%%%%%%%%%%%%%%%%%%%%%%%%%%%%%%%%%%%%%%%%%%%%%%%%%%%%%%%%%%%%%%%%%%
% %%%%%%%%%%%%%%%%%%%%%%%%%%%%%%%%%%%%%%%%%%%%%%%%%%%%%%%%%%%%%%%%%%%%%%%%%%%%%%%
% % APPENDIX
% %%%%%%%%%%%%%%%%%%%%%%%%%%%%%%%%%%%%%%%%%%%%%%%%%%%%%%%%%%%%%%%%%%%%%%%%%%%%%%%
% %%%%%%%%%%%%%%%%%%%%%%%%%%%%%%%%%%%%%%%%%%%%%%%%%%%%%%%%%%%%%%%%%%%%%%%%%%%%%%%
% \newpage
% \appendix
% \onecolumn

% You can have as much text here as you want. The main body must be at most $8$ pages long.
% For the final version, one more page can be added.
% If you want, you can use an appendix like this one, even using the one-column format.
% %%%%%%%%%%%%%%%%%%%%%%%%%%%%%%%%%%%%%%%%%%%%%%%%%%%%%%%%%%%%%%%%%%%%%%%%%%%%%%%
% %%%%%%%%%%%%%%%%%%%%%%%%%%%%%%%%%%%%%%%%%%%%%%%%%%%%%%%%%%%%%%%%%%%%%%%%%%%%%%%

\end{document}